\documentclass{article}

\usepackage{microtype}
\usepackage{graphicx}
\usepackage{subcaption}
\usepackage{booktabs} 
\usepackage{array} 
\usepackage{algorithm}
\usepackage{algorithmic}
\usepackage{hyperref}

\usepackage[preprint]{icml2026}

\usepackage{amsmath}
\usepackage{amssymb}
\usepackage{mathtools}
\usepackage{amsthm}

\usepackage[capitalize,noabbrev]{cleveref}

\theoremstyle{plain}
\newtheorem{theorem}{Theorem}[section]

\theoremstyle{definition}

\newtheorem{assumption}[theorem]{Assumption}
\theoremstyle{remark}

\usepackage{multirow}

\usepackage[textsize=tiny]{todonotes}

\icmltitlerunning{ExpAlign: Expectation-Guided Vision–Language Alignment for Open-Vocabulary Grounding}

\begin{document}

\twocolumn[
  \icmltitle{ExpAlign: Expectation-Guided Vision–Language Alignment for Open-Vocabulary Grounding}



  \icmlsetsymbol{equal}{*}

  \begin{icmlauthorlist}
    \icmlauthor{Junyi Hu}{thu}
    \icmlauthor{Tian Bai}{uestc,comp}
    \icmlauthor{Fengyi Wu}{uestc}
    \icmlauthor{Wenyan Li}{ku}
    \icmlauthor{Zhenming Peng}{uestc}
    \icmlauthor{Yi Zhang}{thu}
  \end{icmlauthorlist}

  \icmlaffiliation{thu}{Department of Automation, Tsinghua University, China}
  \icmlaffiliation{uestc}{University of Electronic Science and Technology of China, China}
  \icmlaffiliation{ku}{University of Copenhagen, Denmark}
  \icmlaffiliation{comp}{Lingsu Lab, China}

  \icmlcorrespondingauthor{Yi Zhang}{zhyi@mail.tsinghua.edu.cn}

  \icmlkeywords{Machine Learning, ICML}

  \vskip 0.3in
]



\printAffiliationsAndNotice{}  

\begin{abstract}
  Open-vocabulary grounding requires accurate vision-language alignment under weak supervision, yet existing methods either rely on global sentence embeddings that lack fine-grained expressiveness or introduce token-level alignment with explicit supervision or heavy cross-attention designs. We propose \textbf{ExpAlign}, a theoretically grounded vision-language alignment framework built on a principled multiple instance learning formulation. ExpAlign introduces an Expectation Alignment Head that performs attention-based soft MIL pooling over token-region similarities, enabling implicit token and instance selection without additional annotations. To further stabilize alignment learning, we develop an energy-based multi-scale consistency regularization scheme, including a Top-K multi-positive contrastive objective and a Geometry-Aware Consistency Objective derived from a Lagrangian-constrained free-energy minimization. Extensive experiments show that ExpAlign consistently improves open-vocabulary detection and zero-shot instance segmentation, particularly on long-tail categories. Most notably, it achieves 36.2 AP$_r$ on the LVIS minival split, outperforming other state-of-the-art methods at comparable model scale, while remaining lightweight and inference-efficient.
\end{abstract}

\section{Introduction}
\label{sec:intro}

Large vision-language models (VLMs) enable powerful zero-shot transfer by aligning images and texts in a shared embedding space~\citep{radford2021learning,li2023blip,jia2021scaling}. 
Despite significant progress, precise spatial grounding, which involves localizing free-form textual concepts within images, remains a key challenge in dense prediction tasks such as open-vocabulary detection and segmentation~\citep{kamath2021mdetr,cai2022x}.
Recent open-vocabulary methods~\citep{liu2024grounding,cheng2024yolo,wang2025yoloe,fu2025wedetect} alleviate vocabulary constraints but often struggle with complex linguistic phenomena, including negation, relations, and compositional descriptions, when fine-grained localization is required.

Recent theoretical analysis reveals an inherent limitation of CLIP-style joint embeddings: collapsing a prompt into a single global representation cannot simultaneously encode attribute binding, spatial relations, and negation under cosine similarity~\citep{kang2025clip}. 
This geometric bottleneck motivates \emph{token-level} vision-language alignment, where informative tokens are selectively emphasized rather than uniformly aggregated. 
However, incorporating token-level reasoning into dense grounding remains nontrivial due to weak supervision and optimization instability.

We propose \textbf{ExpAlign}, an expectation-guided vision-language alignment framework for open-vocabulary grounding.
At its core is the \textbf{Expectation Alignment Head (EAH)}, which produces prompt-conditioned spatial alignment maps by aggregating token-wise similarities through a soft expectation mechanism.
By treating spatial locations as latent instances and textual tokens as competing hypotheses, EAH performs implicit token selection \textit{without instance-level annotations}, admitting a natural interpretation as attention-based soft pooling in multiple instance learning (MIL)~\citep{ilse2018attention}.

To further improve discriminability and spatial coherence, we introduce two auxiliary objectives.
A \textbf{multi-positive InfoNCE loss} enforces prompt-level semantic separation under weak supervision, while a \textbf{Geometry-Aware Consistency Objective (GACO)} regularizes alignment maps by emphasizing relatively consistent regions within each ground-truth mask.
Together, they stabilize optimization and support both positive and negative prompts.

Experiments on LVIS~\citep{gupta2019lvis}, ODinW~\citep{li2022grounded}, and RefCOCO/+/g~\citep{yu2016modeling} show that ExpAlign delivers strong open-vocabulary detection and segmentation performance under similar pre-training scale and model capacity to recent baselines. It achieves competitive or superior results on LVIS rare categories and ODinW subsets, while on RefCOCO/+/g it outperforms detection-focused methods such as YOLOE but trails specialized grounding models like Grounding DINO-T due to CLIP's limited spatial reasoning.

\section{Related Work}

\paragraph{Sentence-level vision-language Alignment.}
vision-language pretraining methods such as CLIP~\citep{radford2021learning} and BLIP~\citep{li2023blip} align whole images with global text embeddings using contrastive objectives. These sentence-level alignment techniques have enabled strong zero-shot transfer for retrieval and classification, and have been adapted to open-vocabulary detection by using prompt embeddings as classifier proxies~\citep{zhou2022detecting}. However, collapsing a prompt into a single vector can lose internal structure and limits fine-grained localization, motivating methods that exploit richer textual and visual interactions.

\paragraph{Token-level and Phrase-level Alignment.}
To capture fine-grained semantics between language and vision, several works explicitly model interactions at the token or phrase level. For instance, GLIP and its variants unify localization and grounding by introducing region-word contrastive alignment and phrase grounding objectives that align phrases with corresponding image regions~\citep{zhang2022glipv2}, enabling the model to learn region–token correspondences beyond global text embeddings. Methods like X-VLM perform multi-grained vision–language pretraining that aligns text with visual concepts at varying granularities, leveraging patch-level or concept-level representations~\citep{zeng2021multi}. Works in temporal grounding also observe that treating all tokens uniformly under cross-modal attention fails to exploit word-level signals crucial for fine-grained alignment~\citep{kang2025empower}.  
Unlike these approaches that rely on explicit cross-attention structures or phrase annotations, ExpAlign uses expectation-based aggregation to softly pool token similarities into spatial alignment maps, preserving token discriminability under weak supervision.

\paragraph{Alignment Regularization and Objectives.}
Contrastive learning remains a core tool for vision–language alignment, with InfoNCE-style objectives widely adopted for separating positive and negative pairs in multimodal settings~\citep{oord2018representation,li2023blip}. Region-level contrastive losses have been proposed to improve localization quality~\citep{zhang2022glipv2, zhong2022regionclip}, and dense alignment objectives have been incorporated into grounding frameworks to better capture spatial semantics. Our multi-positive InfoNCE adapts these ideas to multi-prompt supervision, focusing on the most informative regions. In addition, geometry-aware regularization has been explored in segmentation and structured prediction~\citep{liang2021exploring}, but existing approaches typically rely on absolute geometric cues. In contrast, our geometry-aware consistency objective operates on relative instance statistics, encouraging coherent alignment without rigid spatial targets.

\paragraph{RL-Inspired Regularization in VLMs}
Several works leverage reinforcement learning (RL)-inspired techniques or loss functions as regularizers to improve robustness and generalization in vision-language models. For instance, Group Relative Policy Optimization (GRPO) \citep{shao2024deepseekmath} introduces an efficient PPO variant that computes advantages via group-relative ranking, inspiring advantage-weighted alignment mechanisms. VARP \citep{singh2025varp} uses agent-regularized preferences in RL from VLM feedback to better align rewards and mitigate inaccuracies. PRLL \citep{zhengprll} applies LLM-assisted policy regularization for reward shaping, enabling adaptation in unfamiliar environments. In VLM fine-tuning, RL4VLM \citep{zhai2024fine} directly optimizes VLMs with regularization for decision-making, while VLM-RL \citep{huang2025vlm} incorporates contrastive language goals as regularized rewards in autonomous driving. These works illustrate the increasing adoption of RL-based regularization to enforce consistency and reduce overfitting in multimodal settings.

ExpAlign advances vision–language grounding by combining soft token-level aggregation with principled regularization, balancing expressiveness and optimization stability. It situates itself between sentence-level and structured alignment methods by enabling fine-grained, supervision-efficient alignment without reliance on heavy cross-attention or explicit token annotations.

\begin{figure*}[t]
    \centering
    \includegraphics[width=\linewidth]{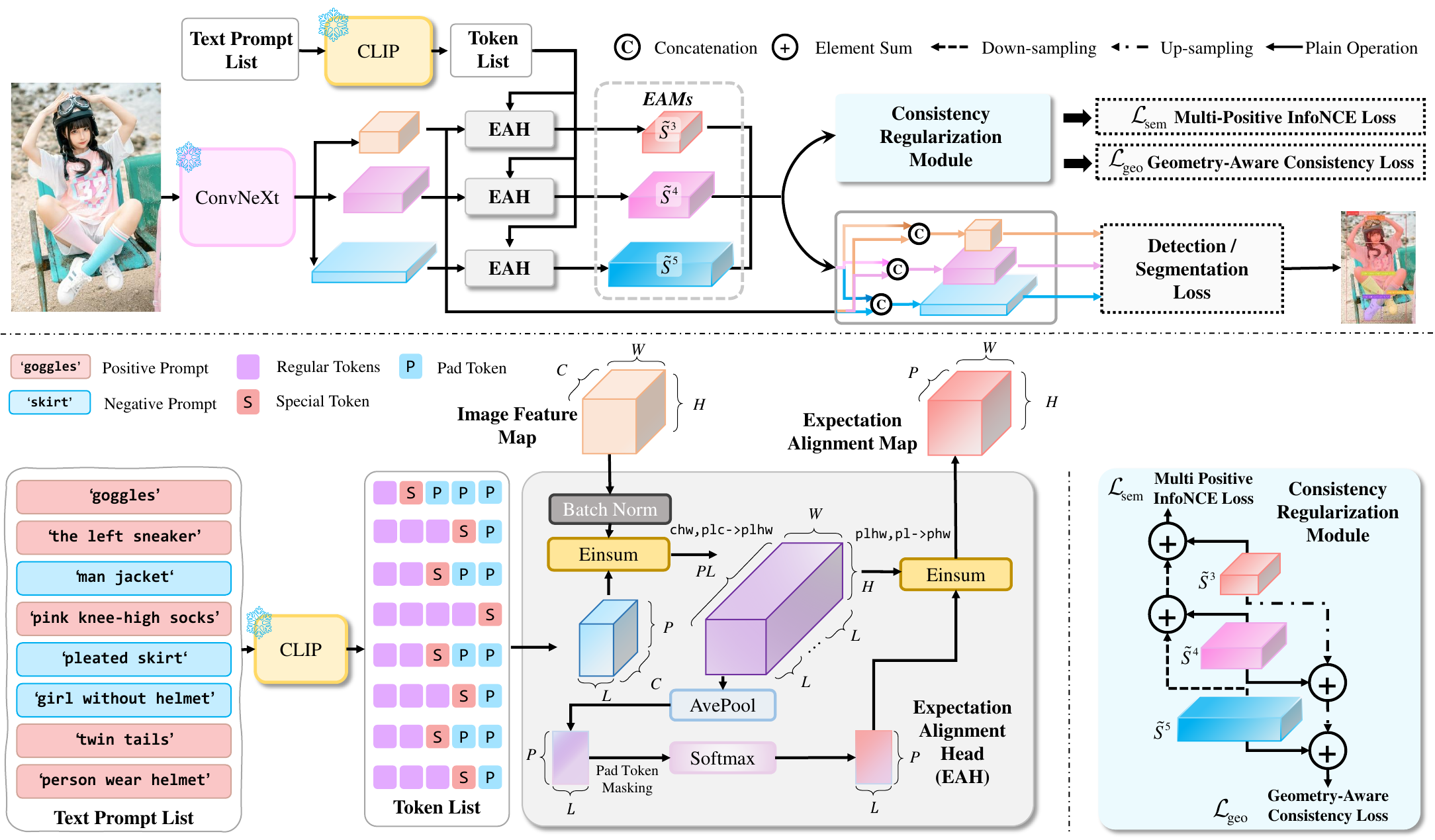}
    \caption{Overview of the proposed ExpAlign framework. \textbf{Top}: the overall pipeline, where prompt-conditioned Expectation Alignment Maps (EAMs) are computed at multiple feature scales and injected into visual features for open-vocabulary grounding and segmentation. \textbf{Bottom-left}: the Expectation Alignment Head, which aggregates token-level vision-language similarities into prompt-specific spatial alignment maps via expectation-based token weighting. \textbf{Bottom-right}: the Consistency Regularization Module, which applies semantic and geometric constraints to regularize the alignment maps. Best viewed in color.}
    \label{fig:overall}
\end{figure*}

\section{Method}

\subsection{Overview}
\label{sec:overview}

We study open-vocabulary grounding, where a model aligns visual regions with flexible language prompts and produces region-level predictions for detection or segmentation. Given an image $I$ and a set of textual prompts $\{T_k\}_{k=1}^{K}$, our goal is to compute prompt-conditioned spatial alignment maps that support robust localization under ambiguous and weak supervision.

 We propose \textbf{ExpAlign}, an expectation-guided vision-language alignment framework. As illustrated in Fig.~\ref{fig:overall}, ExpAlign consists of three key components: (i) an \emph{Expectation Alignment Head (EAH)} that performs soft prompt--region alignment via token-level expectation, producing an \emph{Expectation Alignment Map (EAM)}; (ii) a \emph{Consistency Regularization Module} that enforces cross-scale coherence of alignment maps; and (iii) auxiliary objectives that impose semantic and geometric constraints on the alignment distribution.

This design enables implicit instance selection over both prompt tokens and spatial locations, while remaining fully differentiable and compatible with standard detection and segmentation heads.

\subsection{Expectation Alignment Head}
\label{sec:eah}
For scale $s \in\{3,4,5\}$ the backbone produces a feature map $\mathbf{F}_b^s \in \mathbb{R}^{C\times H_s\times W_s}$, and each prompt $p\in \{1,\cdots, P\}$ (in image $b \in \{1,\cdots, B\}$) is represented by $L$ token embeddings $\mathbf{T}_{b,p} \in \mathbb{R}^{L\times C}$.

We compute the token-wise similarity at every spatial location $(x,y)$:
\begin{equation}
S_{b,p}^s(x,y,l) = \langle \mathbf{F}_b^s(x,y), \, \mathbf{T}_{b,p}(l) \rangle ,
\, l=1,\dots,L.
\end{equation}
To estimate the global relevance of each token, we aggregate spatial evidence by average pooling:
\begin{equation}
\bar S_{b,p}^s(l)
\;=\;
\frac{1}{H_s W_s}
\sum_{x=1}^{H_s}\sum_{y=1}^{W_s}
S_{b,p}^s(x,y,l).
\end{equation}
We then form a token posterior distribution via a softmax over non-pad tokens:
\begin{equation}
\pi_{b,p}^s(l)
\;=\;
\frac{\exp\!\big(\bar S_{b,p}^s(l) / \tau_t \big)}
{\sum_{l'} \exp\!\big(\bar S_{b,p}^s(l') / \tau_t \big)},
\end{equation}
where $\tau_t$ is a temperature parameter.

Finally, we compute the expectation alignment map (EAM) by marginalizing over tokens:
\begin{equation}
\label{eq:EAM}
\tilde S_{b,p}^s(x,y)
\;=\;
\sum_{l=1}^{L}
\pi_{b,p}^s(l)\,
S_{b,p}^s(x,y,l).
\end{equation}

The resulting map $\tilde S_{b,p}^s \in \mathbb{R}^{H_s \times W_s}$ is referred to as the
\emph{Expectation Alignment Map (EAM)} at scale $s$.
This formulation performs implicit token selection by assigning higher weights to globally informative tokens,
while suppressing noisy or irrelevant ones, and yields a spatial alignment score suitable for downstream grounding.

\subsection{Consistency Regularization Module}
\label{sec:consistency}

Given the expectation alignment maps (EAMs) produced at multiple feature scales,
we impose consistency constraints to stabilize vision-language alignment across scales
while respecting their distinct computational and geometric properties.
Specifically, EAMs from different scales are used in a scale-aware manner:
low-resolution maps are employed for efficient semantic alignment aggregation,
whereas high-resolution maps are preserved for geometry-sensitive consistency regularization.


\subsubsection*{Semantic Constraint}
To aggregate semantic evidence across scales while maintaining efficiency,
we first unify multi-scale expectation alignment maps (EAMs) at the coarsest resolution.
Specifically, EAMs from all feature levels are progressively downsampled to the spatial
resolution of the smallest scale (e.g., P5) and summed to obtain a unified alignment map:
\begin{equation}
\tilde S_{b,p}^{\mathrm{dw}} = \left(\operatorname{Down}\left((\operatorname{Down}(\tilde S_{b,p}^3) + \tilde S_{b,p}^4)/2\right) + \tilde S_{b,p}^5\right)/2.
\end{equation}
where $\operatorname{Down}(\cdot)$ denotes resolution-aligned downsampling.

For each image $b$ and prompt $p$, we select the top-$1\%$ highest responses
from the unified map:
\begin{equation}
\mathcal{I}_{b,p}
\;=\;
\operatorname{TopK}\!\left(
\tilde S_{b,p}^{\mathrm{dw}},
\; H_3W_3/100
\right).
\end{equation}
We then define the pooled prompt-level logit as the average alignment score over these
selected locations:
\begin{equation}
\ell_{b,p}
\;=\;
\frac{1}{|\mathcal{I}_{b,p}|}
\sum_{i \in \mathcal{I}_{b,p}}
\tilde S_{b,p}^{\mathrm{dw}}(i).
\end{equation}

Let $\mathcal{P}_b$ denote the set of positive prompts for image $b$.
The multi-positive InfoNCE objective with temperature $\tau$ is formulated as
\begin{equation}
\label{eq:Lcat}
\mathcal{L}_{\text{sem}}
=
-\frac{1}{B}
\sum_{b=1}^{B}
\sum_{p \in \mathcal{P}_b}
\frac{1}{|\mathcal{P}_b|}
\log
\frac{\exp(\ell_{b,p} / \tau)}
{\sum_{p'=1}^{P} \exp(\ell_{b,p'} / \tau)} .
\end{equation}

\subsubsection*{Geometry Constraint}

We introduce a \emph{Geometry-Aware Consistency Objective (GACO)} to regularize the spatial structure of the energy field produced by the Consistency Regularization Module.
Instead of enforcing absolute geometric targets, GACO shapes the energy landscape through relative, instance-wise consistency within the ground-truth region.

we construct a unified high-resolution
Expectation Alignment Map by progressively aggregating EAMs from coarse to fine levels.
Starting from the coarsest scale (P5), we apply a top-down fusion strategy:
\begin{equation}
    \tilde S_{b,p}^{\mathrm{up}} = \left(\operatorname{Up}\left((\operatorname{Up}(S_{b,p}^5) +  S_{b,p}^4)/2\right) + S_{b,p}^3\right)/2.
\end{equation}
where $\operatorname{Up}(\cdot)$ denotes resolution-aligned upsampling. This aggregation preserves fine-grained geometry while incorporating multi-scale semantic evidence.

Given the aggregated alignment map, we define a normalized distribution over
prompt--patch pairs by applying a softmax over all prompts and spatial locations:
\begin{equation}
\mathbb{P}_b(p,i)
=
\frac{\exp\!\big(\tilde S_{b,p}^{\mathrm{up}}(i)\big)}
{\sum\limits_{p'=1}^{P} \sum\limits_{i' \in \Omega}
\exp\!\big(\tilde S_{b,p'}^{\mathrm{up}}(i')\big)},
\end{equation}
where $i$ indexes spatial locations at the P3 resolution. This distribution assigns higher probability mass to patches that are more strongly aligned with a given prompt.

Let $M_{b,p}(i)\in\{0,1\}$ denote the binary ground-truth mask associated with prompt $p$,
resized to the P3 resolution, and define the positive region
$\mathcal{M}_{b,p}=\{\, i \mid M_{b,p}(i)=1 \,\}$.
We introduce a bounded local alignment confidence
\begin{equation}
R_b(p,i) = \sigma\!\big(\tilde S_{b,p}^{\mathrm{up}}(i)\big),
\end{equation}
and compute its mean and standard deviation over the positive region:
\begin{equation*}
\begin{aligned}
\mu_{b,p} &=
\frac{1}{|\mathcal{M}_{b,p}|}
\sum_{i\in\mathcal{M}_{b,p}} R_b(p,i),\\
\sigma_{b,p} &=
\sqrt{
\frac{1}{|\mathcal{M}_{b,p}|}
\sum_{i\in\mathcal{M}_{b,p}}
\big(R_b(p,i)-\mu_{b,p}\big)^2
+ \varepsilon }.
\end{aligned}
\end{equation*}

Based on these statistics, we define a \emph{relative consistency score}
\begin{equation}
A_{b,p}(i)
=
\operatorname{clip}\!\left(
\frac{R_b(p,i)-\mu_{b,p}}{\sigma_{b,p}},
\,-c,\,c
\right),
\end{equation}
which measures how well each patch agrees with the dominant spatial structure of its corresponding ground-truth instance. Importantly, this score depends only on intra-instance relative statistics and does not impose absolute alignment targets.

The final geometry-aware consistency loss is defined as
\begin{equation}
\label{eq:Lgeo}
\!\!\mathcal{L}_{\text{geo}}
\!=\!
\!-\!
\frac{1}{\!\!\sum_b\!\!\sum_p |\mathcal{M}_{b,p}|}
\!\!\sum_{b=1}^{B}
\!\sum_{p=1}^{P}
\!\sum_{i\in\mathcal{M}_{b,p}}
\!\!\!A_{b,p}(i)\,
\!\log\!\mathbb{P}_b(p,i).
\end{equation}

This objective redistributes probability mass within each ground-truth region according to relative geometric consistency, encouraging spatially coherent alignment maps while remaining invariant to monotonic transformations of alignment scores. Rather than collapsing responses via pointwise regression, GACO sculpts the geometry of alignment maps and naturally supports implicit instance selection.

Notably, under the Gibbs energy minimization framework with the three assumptions detailed in Appendix~\ref{app:variational}, the two loss terms in this module are mathematically derived as principled regularizers for cross-modal alignment. The full variational derivation is provided in Appendix.

\subsection{Full Training Objective}
The proposed consistency regularization serves as auxiliary supervision during training, complementing the standard detection or segmentation objective. Specifically, the semantic consistency loss $\mathcal{L}_{\text{sem}}$ promotes instance-level prompt--patch alignment, while the geometry-aware consistency loss $\mathcal{L}_{\text{geo}}$ encourages spatially coherent responses within each ground-truth region. Let $\mathcal{L}_{\text{det/seg}}$ denote the task-specific loss (including classification, regression, and mask prediction terms when applicable). The overall training objective is
\[
\mathcal{L}
=
\mathcal{L}_{\text{det/seg}}
+
\lambda_{\text{sem}}\,\mathcal{L}_{\text{sem}}
+
\lambda_{\text{geo}}\,\mathcal{L}_{\text{geo}}.
\]
Both consistency losses are used only during training and are discarded at inference time, preserving standard prediction behavior while improving vision-language alignment and spatial consistency.

\subsection{Connection to Multiple Instance Learning}
\label{sec:mil}

Our Expectation Alignment Map (EAM) is mathematically equivalent to attention-based soft pooling in multiple instance learning \citep{ilse2018attention}. Concretely, by treating each spatial location as an instance and each textual prompt as a bag, the EAM implements a soft-MIL pooling over token hypotheses, which grants permutation invariance and expressive pooling power. The multi-positive InfoNCE loss and the Geometry-Aware Consistency Objective (GACO) further enforce discriminative bag semantics and regularize instance relationships within positive regions. A formal proof is provided in Appendix~\ref{app:mil_supp}.

\begin{table*}[t]
\setlength{\abovecaptionskip}{3pt} 
\setlength{\belowcaptionskip}{3pt}
\renewcommand\arraystretch{1.07}
  \centering
  \caption{\textbf{Zero-shot detection performance.} Metrics on LVIS val~\cite{gupta2019lvis} and minival~\cite{kamath2021mdetr} are fixed AP~\cite{achal2022evaluating}. All models use an input resolution of 640×640, except for those with Swin-Tiny as the backbone, which employ 800×1333. For training data, OG indicates Objects365~\cite{shao2019objects365} and GoldG~\cite{kamath2021mdetr}. RefC indicates RefCOCO/g/+~\cite{yu2016modeling}.}
  \resizebox{\linewidth}{!}{
      \begin{tabular}{l|lcc|cccc|cccc|c|c}
        \hline
        \multirow{2}{*}{Method} & \multirow{2}{*}{Backbone} & \multirow{2}{*}{Pre-train Data} & \multirow{2}{*}{\#Params}  & \multicolumn{4}{c|}{LVIS$^{\text{minival}}$}  & \multicolumn{4}{c|}{LVIS}  & ODinW13 & ODinW35 \\
        & & & & AP & AP$_{r}$ & AP$_{c}$ & AP$_{f}$ & AP & AP$_{r}$ & AP$_{c}$ & AP$_{f}$ & AP & AP \\
        \hline
        GLIP-T~\cite{li2022grounded} & Swin-T & OG & 232M & 24.9 & 17.7 & 19.5 & 31.0 & 16.5 & 7.5 & 11.6 & 26.1 & - & - \\
        DetCLIP~\cite{yao2022detclip} & Swin-T & OG & -  & 34.4 & 26.9 & 33.9 & 36.3 & - & - & - & - & - & - \\
        GDINO-T~\cite{liu2024grounding} & Swin-T & OG, Cap4M & 172M & 27.4 & 18.1 & 23.3 & 32.7 & - & - & - & - & \textbf{49.7} & 22.3 \\
        OVLW-DETR-L~\cite{wang2024ovlw} & LW-DETR-L & OG & 47M  & 33.5 & 26.5 & 33.9 & 34.4 & - & -  & - & - & - & - \\
        OmDet-Turbo-B~\cite{zhao2024omdet} & ConvNeXt-B & OG & 175M  & 34.7 & - & - & - & - & -  & - & - & - & - \\
        YOLO-Worldv2-L~\cite{cheng2024yolo} & YOLOv8-L & OG & 48M & 35.4 & 27.6 & 34.1 & 38.0  & 26.8  & 19.8 & 23.6 & 33.4   & 38.4 & 17.1 \\
        GDINO 1.5 Edge~\cite{ren2024grounding} & EfficientViT-L1 & Grounding-20M & -  & 33.5 & 28.0 & 34.3 & 33.9 & 27.3 & 26.3 & 25.7 & 29.6 & - & - \\
        YOLOE-8-L~\cite{wang2025yoloe} & YOLOv8-L & OG & 45M  & 35.9 & 33.2 & 34.8 & 37.3 & - & -  & - & - & - & - \\
        \hline
        ExpAlign (Ours) & ConvNeXt-T & OG & 60M & \textbf{37.2} & 35.8 & \textbf{37.2} & \textbf{37.6} & \textbf{30.3} & \textbf{26.5} & \textbf{29.8} & \textbf{33.7} & 48.0 & \textbf{22.6} \\
        ExpAlign (Ours) & ConvNeXt-T & OG,RefC & 60M & 37.1 & \textbf{36.2} & 37.1 & 37.4 & 29.5 & 24.8 & 28.0 & 33.4 & 47.7 & 22.4 \\
        \hline
      \end{tabular}
    }
  \label{tab:lvis_result}
\end{table*}

\section{Experiment}
\subsection{Implementation Details}

\textbf{Model.}
ExpAlign is implemented as a lightweight vision-language alignment module that can be seamlessly integrated into standard multi-scale detection and segmentation architectures.
Unless otherwise specified, we adopt a frozen DINOv3~\cite{simeoni2025dinov3} ConvNeXt-T image encoder as the visual backbone.
Following the encoder, we employ the same YOLOv8~\cite{varghese2024yolov8} FPN-style feature enhancement module to produce multi-scale feature maps.
The detection and segmentation heads, along with their corresponding loss functions, strictly follow the standard YOLOv8 formulation without modification.

Text prompts are encoded using a frozen CLIP~\cite{radford2021learning} ViT-L/14 text encoder, where we retain all token-level representations before the end-of-text (EOT) token, rather than collapsing the prompt into a single global embedding.
To map textual tokens into the same feature space as visual representations, we append a lightweight Residual SwiGLU feed-forward network (SwiGLUFFN)~\cite{shazeer2020glu} after the CLIP text encoder.
The second linear layer of the SwiGLUFFN is initialized to zero, such that the module initially behaves as an identity mapping.
This design stabilizes early training and ensures that token-level alignment is learned progressively without disrupting the pretrained CLIP geometry.

The Expectation Alignment Head (EAH) is attached to each feature level and computes prompt-conditioned alignment maps via token-wise similarity aggregation.
The Consistency Regularization Module operates solely on the resulting alignment maps and introduces no additional learnable parameters.

\textbf{Data.}
We adopt the same data protocol as~\citet{cheng2024yolo} and train ExpAlign on a combination of detection and grounding datasets.
Specifically, we use Objects365~\cite{shao2019objects365} for large-scale object detection and GoldG~\cite{kamath2021mdetr}, which aggregates GQA~\cite{hudson2019gqa} and Flickr30k~\cite{plummer2015flickr30k}, for vision-language grounding.
To avoid data leakage, all images overlapping with COCO~\cite{lin2014microsoft} are excluded from the training set. Since pixel-level annotations are not available for most training images, we generate pseudo instance masks for segmentation by applying the SAM-2.1 model~\cite{ravi2024sam} to ground-truth bounding boxes from the detection and grounding datasets.

\textbf{Training.}
All experiments employed the AdamW optimizer combined with a cosine learning rate scheduler across a two-stage training procedure. Both training and evaluation were carried out on a dedicated machine featuring eight NVIDIA RTX Pro 6000 GPUs, each with 96 GB of memory.
Both the image encoder and the text encoder remain frozen throughout training.
In the first stage, the model is trained for 30 epochs using only the standard YOLOv8 detection and segmentation losses, with an initial learning rate $\text{lr}_0=0.002$, final learning rate ratio $\text{lrf}=0.1$, and a warmup of 3 epochs.
In the second stage, we enable the multi-positive InfoNCE loss and the Geometry-Aware Consistency Objective (GACO), and continue training for another 20 epochs with a reduced initial learning rate $\text{lr}_0=0.001$, $\text{lrf}=0.2$, and no warmup.
The loss weights are fixed to $\lambda_{\text{sem}}=0.5$ and $\lambda_{\text{geo}}=1.0$ across all experiments.
The semantic contrastive loss is applied at the lowest feature resolution for efficiency, while GACO is computed on high-resolution alignment maps to preserve spatial structure.

\subsection{Zero-shot Detection and Segmentation Performance}

ExpAlign exhibits competitive zero-shot open-vocabulary detection performance under fair pre-training and inference conditions. We perform zero-shot evaluation on the val and minival splits of LVIS with fixed AP protocol. LVIS features 1203 classes with a long-tail distribution, while ODinW spans 35 diverse real-world datasets, testing generalization to varied domains and vocabularies.

As shown in Table \ref{tab:lvis_result}, ExpAlign with OG and RefC achieves 37.1 AP on LVIS minival and leads in rare-category performance with AP$_r$ of 36.2. On full LVIS val, it reaches 29.5 AP and 24.8 AP$_r$, benefiting from RefCOCO referring expression supervision for improved long-tail handling. On ODinW, it attains 47.7 AP on ODinW13 and 22.4 AP on ODinW35, substantially outperforming GLIP-T and closely matching or exceeding Grounding DINO-T, using a lightweight design with only 60M total parameters (26M frozen), which enables superior efficiency under comparable backbone scale.

All models use 640×640 input resolution except those with Swin-Tiny backbone, which use 800×1333. These results demonstrate the effectiveness of referring expression data for enhancing rare-object detection and real-world robustness without relying on massive extra pre-training data.

\begin{table}[t]
\centering
\caption{\textbf{Zero-shot instance segmentation performance} on LVIS \texttt{val} set using standard mask AP$^m$. ExpAlign  and YOLOE are evaluated purely zero-shot without any LVIS images or annotations during training. In contrast, YOLO-Worldv2-L is fine-tuned on LVIS-Base data for the segmentation head.}
\label{tab:lvis-segm}
\small
\setlength{\tabcolsep}{2.5pt}
\resizebox{\linewidth}{!}{%
\begin{tabular}{p{4.2cm}<{\raggedright}|p{1cm}<{\centering}p{1cm}<{\centering}p{1cm}<{\centering}p{1cm}<{\centering}} 
\toprule
Model  & AP$^m$ & AP$_r^m$ & AP$_c^m$ & AP$_f^m$ \\ 
\midrule
YOLO-Worldv2-L  & 19.8 & 17.2 & 17.5 & 23.6 \\
OpenSeeD~\cite{zhang2023simple}  & 21.0 & - & - & - \\
YOLOE-v8-L  & 23.5  & 21.9  & 21.6  & 26.4  \\
YOLOE-11-L  & 22.6  & 19.3 & 20.9  & 26.0  \\
\midrule
ExpAlign    &  \textbf{29.9}  & 29.0 & \textbf{30.9}  & \textbf{29.1}  \\
ExpAlign ( + RefC) & 29.8  & \textbf{29.7} & 30.8  & 28.9  \\
\bottomrule
\end{tabular}
}
\end{table}

Furthermore, as shown in \cref{tab:lvis-segm}, ExpAlign achieves strong zero-shot instance segmentation on the LVIS \texttt{val} set using the standard AP$^m$ metric. It attains 29.9 AP$^m$ overall and 29.0 AP$^m_r$ on rare categories without any exposure to LVIS images during training. This performance far surpasses YOLO-Worldv2-L fine-tuned on LVIS-Base at 19.8 AP$^m$ and YOLOE variants ranging from 22.6 to 23.5 AP$^m$. The substantial improvement of 6 to 10 AP$^m$ is largely attributed to the GACO regularization term introduced during pre-training, which significantly enhances mask precision and boundary alignment across long-tail categories in open-vocabulary settings.

\subsection{Downstream Transferring}
We evaluate ExpAlign's downstream transferability on the COCO dataset through fine-tuning for object detection and instance segmentation, as shown in Table~\ref{tab:coco-tf}. Under linear probing (backbone frozen, 10 epochs), ExpAlign outperforms YOLOE-v8-L and YOLOE-11-L in both bounding-box AP$^b$ and mask AP$^m$. In full tuning (80 epochs, all parameters trainable), ExpAlign further surpasses these baselines across most metrics, including AP$^b_{50}$, AP$^m_{50}$, and AP$^m_{75}$. These consistent gains across both strategies highlight that ExpAlign's pre-training design enables more efficient and effective adaptation to standard supervised tasks compared to recent open-vocabulary baselines.

\begin{table}[!t]
\centering
\caption{\textbf{Downstream fine-tuning performance on COCO.} ExpAlign is fine-tuned on the COCO train2017 set and evaluated on val2017 using standard bounding-box AP$^b$ and mask AP$^m$ metrics, including AP at IoU thresholds 0.50 and 0.75. We compare two practical strategies: linear probing with the backbone frozen for 10 epochs and full tuning with all parameters trainable for 80 epochs. Training-from-scratch baselines are included for reference.}
\label{tab:coco-tf}
\small
\renewcommand\arraystretch{1.1}
\setlength{\tabcolsep}{3pt}
\resizebox{\linewidth}{!}{%
\begin{tabular}{p{2.1cm}<{\raggedright}|p{0.8cm}<{\centering}p{0.8cm}<{\centering}p{0.8cm}<{\centering}p{0.8cm}<{\centering}p{0.8cm}<{\centering}p{0.8cm}<{\centering}p{0.8cm}<{\centering}} 
\toprule
Model & Epochs & AP$^b$ & AP$^b_{50}$ & AP$^b_{75}$ & AP$^m$ & AP$^m_{50}$ & AP$^m_{75}$ \\ 
\hline
\multicolumn{8}{l}{\textit{Training from scratch}} \\
YOLOv8-L & 300 & 52.4 & 69.3 & 57.2 & 42.3 & 66.0 & 44.9  \\
YOLO11-L & 600 & 53.3 & 70.1 & \textbf{58.2} & 42.8 & 66.8 & \textbf{45.5}  \\ 
\hline
\multicolumn{8}{l}{\textit{Linear probing}} \\
YOLOE-v8-L & 10 & 45.4 & 63.3 & 50.0 & 38.3 & 59.6 & 40.8 \\
YOLOE-11-L & 10 & 45.1 & 62.8 & 49.5 & 38.0 & 59.2 & 40.6 \\ 
ExpAlign (Ours)   & 10 & 47.2 & 65.3 & 51.7 & 39.2 & 61.7 & 41.8 \\ 
\hline
\multicolumn{8}{l}{\textit{Full tuning}} \\
YOLOE-v8-L & 80 & 53.0 & 69.8 & 57.9 & 42.7 & 66.5 & 45.6 \\
YOLOE-11-L & 80 & 52.6 & 69.7 & 57.5 & 42.4 & 66.2 & 45.2 \\
ExpAlign (Ours) & 80 & \textbf{53.5} & \textbf{70.7} & 57.8 & \textbf{42.9} & \textbf{67.0} & \textbf{45.5} \\ 
\bottomrule
    \end{tabular}}
    \vspace{-0.3cm}
\end{table}

\subsection{Referring Expression Comprehension Performance}
As shown in Table~\ref{tab:ref}, ExpAlign underperforms significantly on referring expression comprehension tasks compared to Grounding DINO-T, even when pre-trained with the same RefC data. For example, ExpAlign with RefC achieves only 51.6/59.3/47.7 on RefCOCO splits and 65.6/64.0 on RefCOCOg, far below Grounding DINO-T's 74.0/74.9/59.3 and 71.1/72.1. We acknowledge this limitation openly. The primary cause is likely the CLIP text encoder's inherent weakness in understanding positional and relational language (e.g., ``left of'', ``behind'', ``next to''), which is crucial for many referring expressions, especially on RefCOCO+ and RefCOCOg. In contrast, Grounding DINO benefits from a more specialized text encoder and fusion mechanism better suited for spatial reasoning. This highlights a key area for future improvement in ExpAlign's design.

\begin{table}[!t]
\centering
\renewcommand\arraystretch{1.15}
  \caption{\textbf{Performance on common referring expression comprehension} datasets. The evaluation metric for RefCOCO, RefCOCO+, and  RefCOCOg is the Top-1 accuracy. * indicates removed mosaic, flip, and HSV augmentations in phase-2 training.}
  \resizebox{1\linewidth}{!}{
      \begin{tabular}{l|c|cccccccc}
        \hline
        \multirow{2}{*}{Method} & \multirow{2}{*}{Pre-Train Data}  & \multicolumn{3}{c}{RefCOCO} & \multicolumn{3}{c}{RefCOCO+} & \multicolumn{2}{c}{RefCOCOg}  \\
        & & val & testA & testB & val & testA & testB & val & test \\
        \hline
        YOLOE-8-L & OG & 32.8 & 43.7 & 45.7 & 28.5 & 31.1 & 45.8 & 22.6 & 23.1  \\
        GDINO-T & OG & 50.4 & 57.2 & 43.2 & 51.4 & 57.6 & 45.8 & 67.5 & 67.1  \\
        GDINO-T & OG, RefC & 74.0 & 74.9 & 59.3 & 66.8 & 69.9 & 56.1 & 71.1 & 72.1  \\
        ExpAlign & OG & 42.2 & 48.2 & 46.7 & 37.62 & 35.3 & 45.2 & 60.0 & 59.67  \\
        ExpAlign* & OG, RefC & 51.6 & 59.3 & 47.7 & 48.9 & 47.5 & 45.5 & 65.6 & 64.0  \\
        \hline
      \end{tabular}
    }
\label{tab:ref}
\end{table}
\subsection{Ablation Study}
We further provide extensive analyses for the effectiveness of designs in our ExpAlign. Experiments are conducted on fixed AP~\cite{achal2022evaluating} is reported on LVIS minival splits set for zero-shot evaluation, by default. 

As shown in Table~\ref{tab:ablation_token}, the token-level alignment strategy (EAH) significantly outperforms simpler representations. Compared to mean pooling and global pooled token (EOT), EAH improves 5.2 and 2.7 absolute AP points, respectively. More specifically, on rare categories, applying EAH reaches 36.2 AP$_r$, achieving the largest gain of 8.9 AP$_r$ points compared to mean pooling. This demonstrates that explicit alignment at the token level captures finer-grained cross-modal correspondence, leading to better generalization on long-tail distributions.

\begin{table}[t]
\centering
\caption{\textbf{Ablation study} on token-level alignment versus pooled token representations on the LVIS dataset.}
\label{tab:ablation_token}
\resizebox{1\linewidth}{!}{
\begin{tabular}{p{4.5cm}<{\raggedright}|p{1cm}<{\centering}p{1cm}<{\centering}p{1cm}<{\centering}p{1cm}<{\centering}}
\toprule
Alignment Strategy & AP & AP$_\text{r}$ & AP$_\text{c}$ & AP$_\text{f}$ \\
\midrule
Mean pooling over tokens  & 31.9 & 27.3 & 30.5 & 32.2 \\
Global pooled token (EOT) & 34.4 & 33.2 & 34.5 & 35.4 \\
Token-level alignment (EAH)  & \textbf{37.1} & \textbf{36.2} & \textbf{37.1} & \textbf{37.4} \\
\bottomrule
\end{tabular}
}
\end{table}

Table~\ref{tab:ablation_lambda} ablates the loss weights for semantic contrastive loss ($\lambda_{\text{sem}}$) and geometry-aware consistency ($\lambda_{\text{geo}}$). Using $\lambda_{\text{sem}}=0.5$ alone reaches 37.0 AP and strong rare performance (35.8 AP$_r$). Adding $\lambda_{\text{geo}}$ (especially at 0.5 or 1.0) consistently improves overall AP and frequent/common categories, with the best results at $\lambda_{\text{sem}}=0.5$ + $\lambda_{\text{geo}}=0.5$ (37.8 AP) or $\lambda_{\text{sem}}=0.0$ + $\lambda_{\text{geo}}=1.0$ (37.1 AP, 35.6 AP$_r$). Excessive $\lambda_{\text{sem}}$ tends to hurt rare-category performance when combined with high $\lambda_{\text{geo}}$, indicating a necessary balance. These results confirm that the geometry-aware term complements semantic alignment by enforcing spatial consistency across the long tail.

\begin{table}[t]
\renewcommand\arraystretch{0.95}
\centering
\caption{\textbf{Ablation study on the loss weights} $\lambda_{\text{sem}}$ (semantic contrastive loss) and $\lambda_{\text{geo}}$ (geometry-aware consistency objective) on the LVIS dataset.}
\label{tab:ablation_lambda}
\resizebox{1\linewidth}{!}{
\begin{tabular}{p{1.2cm}<{\centering}p{1.2cm}<{\centering}|p{1.2cm}<{\centering}p{1.2cm}<{\centering}p{1.2cm}<{\centering}p{1.2cm}<{\centering}}
\toprule
$\lambda_{\text{sem}}$ & $\lambda_{\text{geo}}$ 
& AP & AP$_\text{r}$ & AP$_\text{c}$ & AP$_\text{f}$ \\
\midrule
0.0 & 0.0 & 35.0 & 32.0 & 35.8 & 35.1 \\
0.5 & 0.0 & 37.0 & 35.8 & 35.9 & 38.1 \\
1.0 & 0.0 & 36.1 & 30.2 & 33.1 & 36.7 \\
\midrule
0.0 & 0.5 & 37.6 & 31.7 & \textbf{37.7} & \textbf{38.6} \\
0.5 & 0.5 & \textbf{37.8} & 32.1 & 37.6 & 38.3 \\
1.0 & 0.5 & 36.2 & 29.2 & 26.1 & 37.6 \\
\midrule
0.0 & 1.0 & 37.1 & 35.6 & 37.6 & 37.2 \\
0.5 & 1.0 & 37.2 & \textbf{35.9 }& 37.2 & 37.6 \\
1.0 & 1.0 & 36.9 & 32.3 & 36.3 & 37.7 \\
\midrule
0.0 & 1.5 & 36.6 & 34.5 & 36.2 & 37.4 \\
0.5 & 1.5 & 36.7 & 35.0 & 36.6 & 37.2 \\
1.0 & 1.5 & 35.1 & 31.9 & 35.4 & 36.0 \\
\bottomrule
\end{tabular}
}
\vspace{-0.2cm}
\end{table}

Table~\ref{tab:ablation_backbone} compares different backbones under identical detection and segmentation heads on LVIS$^{\text{minival}}$. The YOLOv8 backbone achieves 35.6 AP overall, with 33.9 AP$_r$ on rare categories. In contrast, using the DINOv3 backbone without freezing leads to training collapse (indicated by N/A), resulting in no meaningful convergence. However, when the DINOv3 backbone is frozen during pre-training, ExpAlign reaches 37.2 AP overall and 35.9 AP$_r$, outperforming YOLOv8 by 1.6 AP and showing particular gains on rare categories. This suggests that preserving the rich, high-quality pre-trained features from a strong frozen vision foundation model is crucial for ExpAlign's cross-modal alignment objectives, whereas a detection-oriented backbone like YOLOv8 or unfrozen DINOv3 hinders effective learning of the alignment signals.

\begin{figure*}[!htbp]
\setlength{\abovecaptionskip}{3pt} 
\setlength{\belowcaptionskip}{3pt}
    \centering
    \includegraphics[width=1\linewidth]{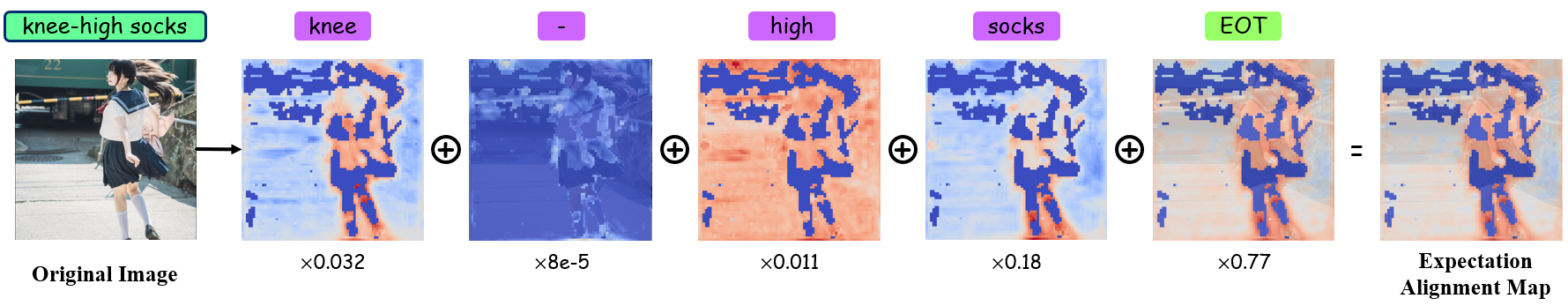}
    \caption{Expectation alignment map calculation diagram. Spatial alignment maps are first computed for individual text tokens. All maps are then  aggregated with their importance weight (displayed below each map) to form a prompt-conditioned expectation alignment map.}
    \label{fig:EAMvis}
\end{figure*}

\begin{figure*}[!htbp]
\setlength{\abovecaptionskip}{3pt} 
\setlength{\belowcaptionskip}{3pt}
    \centering
    \begin{subfigure}[t]{0.245\textwidth}
        \centering
        \includegraphics[width=\linewidth]{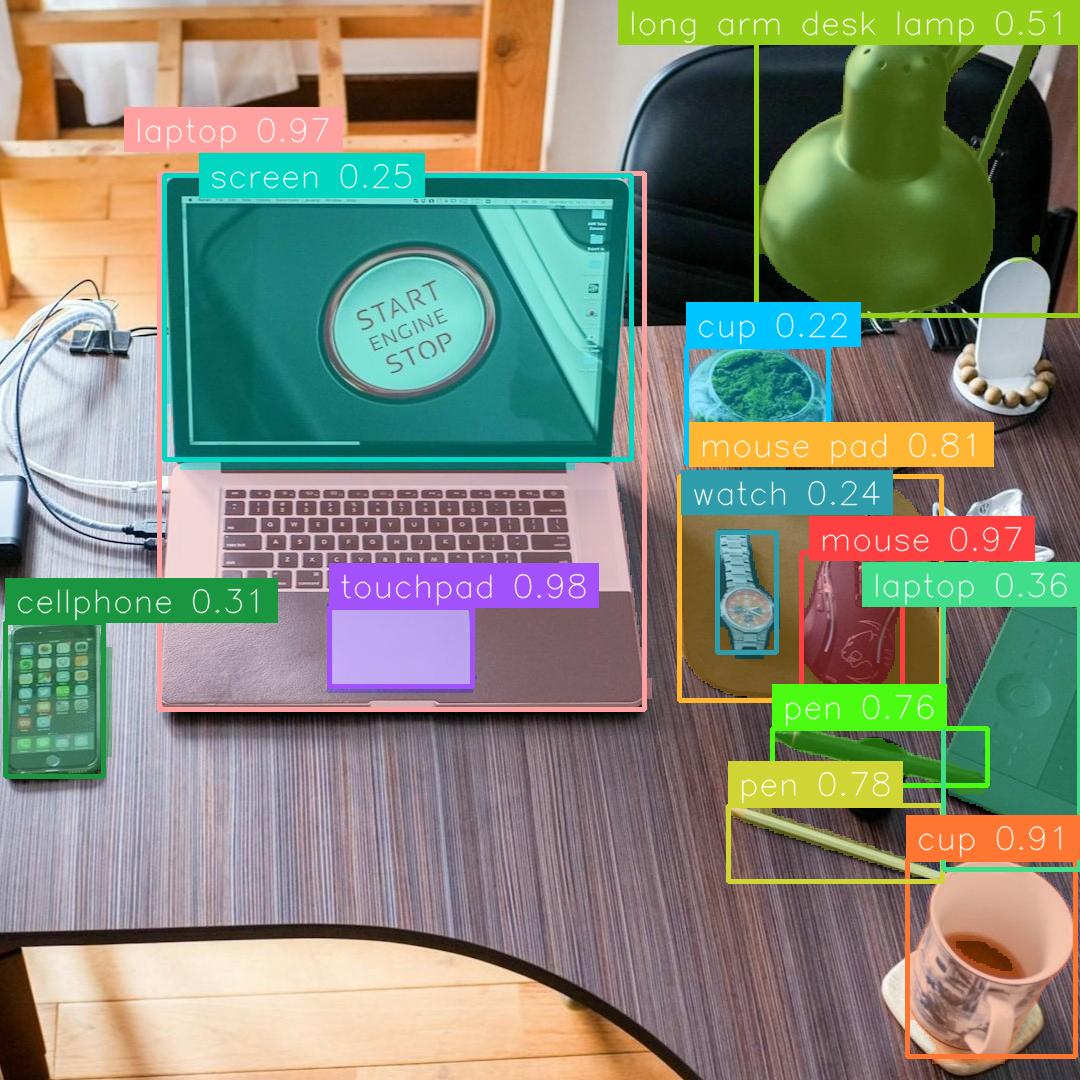}
        \caption{}
        \label{fig:visuala}
    \end{subfigure}\hfill
    \begin{subfigure}[t]{0.245\textwidth}
        \centering
        \includegraphics[width=\linewidth]{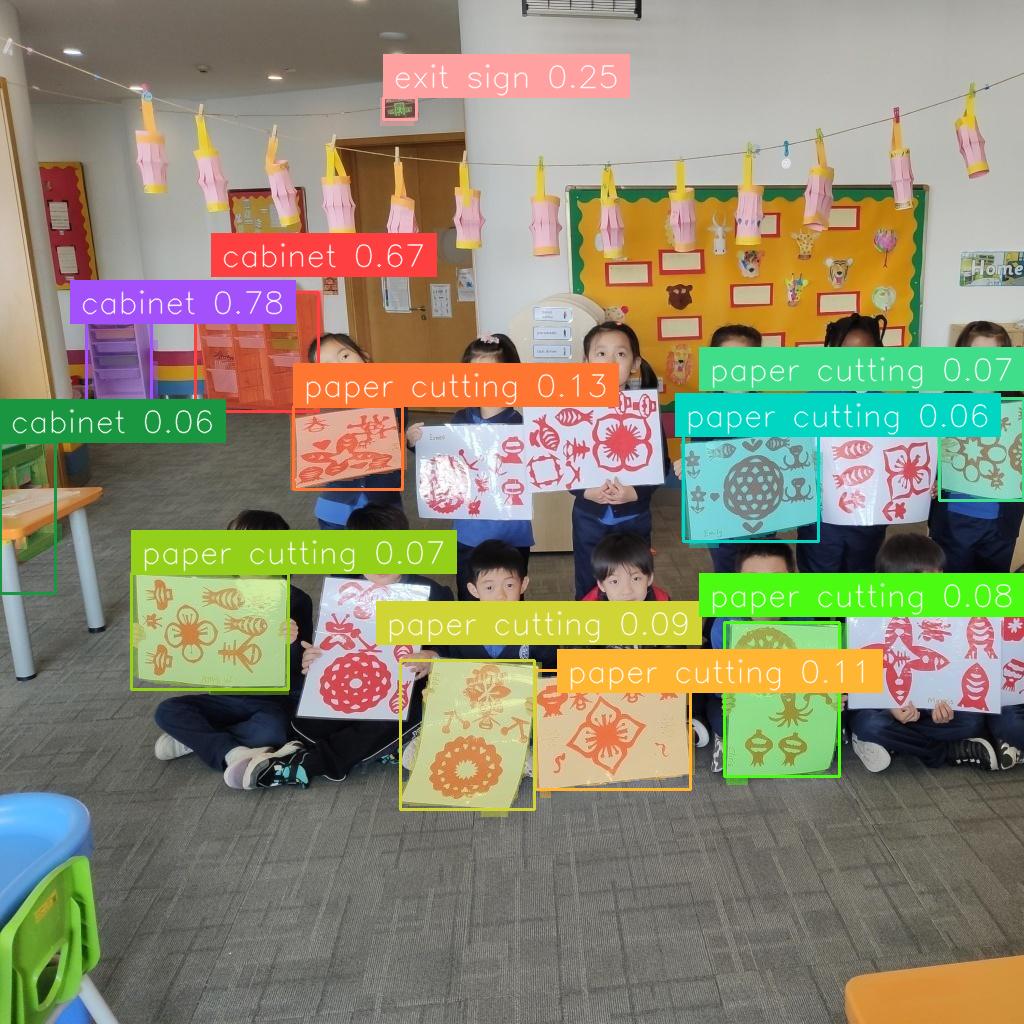}
        \caption{}
        \label{fig:visualb}
    \end{subfigure}\hfill
    \begin{subfigure}[t]{0.245\textwidth}
        \centering
        \includegraphics[width=\linewidth]{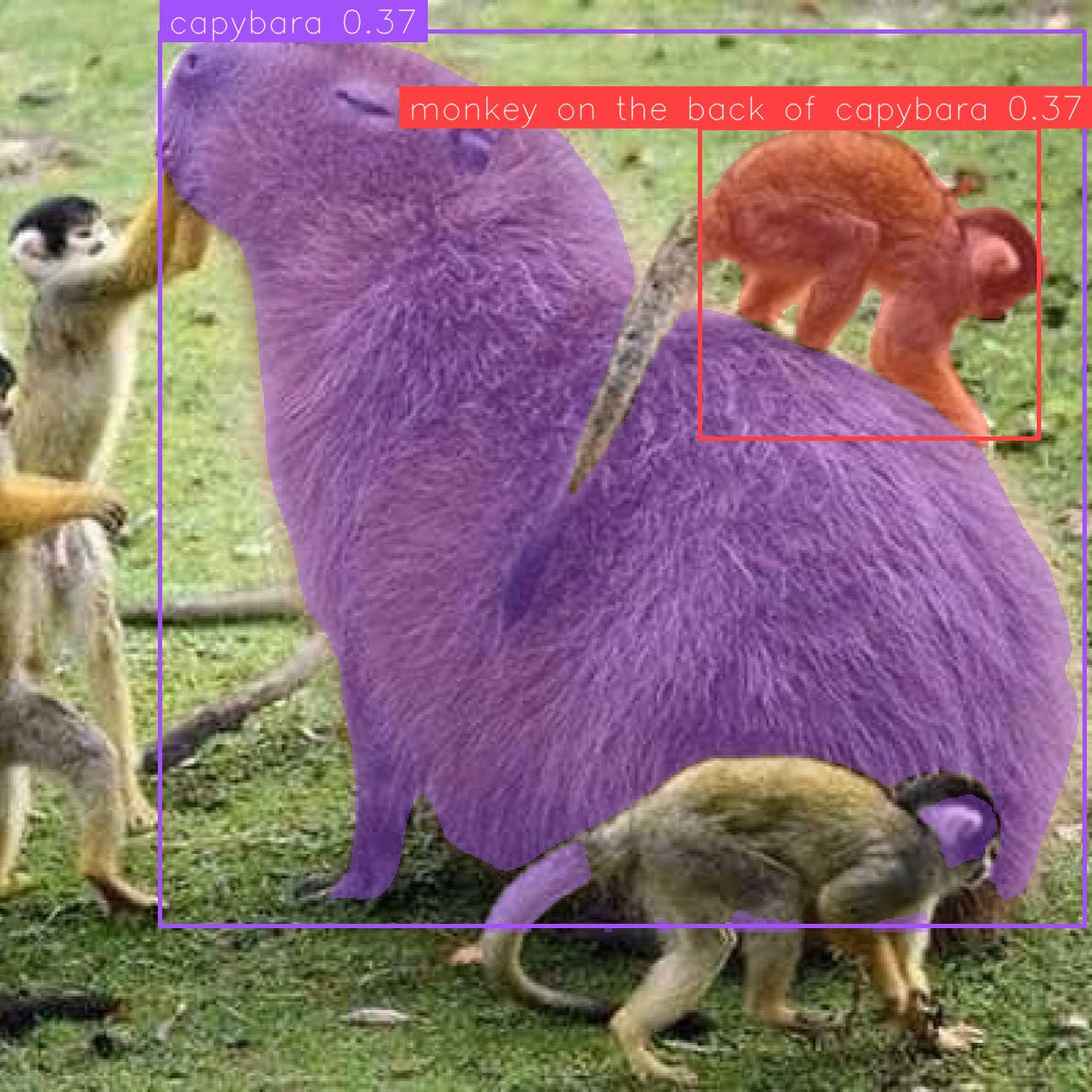}
        \caption{}
        \label{fig:visualc}
    \end{subfigure}\hfill
    \begin{subfigure}[t]{0.245\textwidth}
        \centering
        \includegraphics[width=\linewidth]{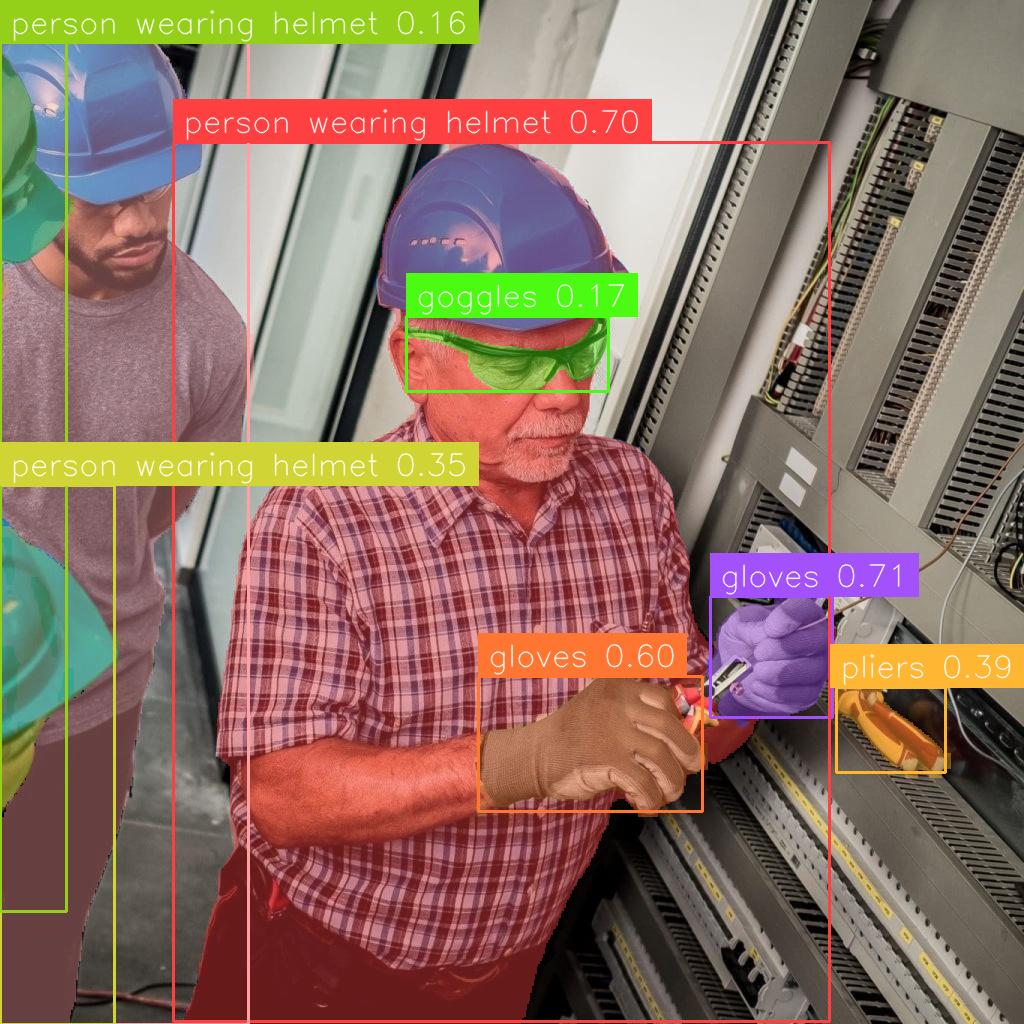}
        \caption{}
        \label{fig:visuald}
    \end{subfigure}
    \caption{Qualitative examples of detection and segmentation results. (a) prompts: laptop, cellphone, watch, cup, mouse, long arm desk lamp, pen, mouse pad, touchpad, screen, keyboard. (b) prompts: paper cutting, cabinet, exit sign. (c) prompts: capybara, monkey on the back of capybara. (d) prompts: person wearing helmet, pliers, gloves, goggles. Zoom in for better visual effect.}
    \label{fig:visual}
\end{figure*}

\begin{table}[t]
\centering
\caption{\textbf{Backbone comparison} on the LVIS dataset. All settings use the same detection and segmentation heads.}
\label{tab:ablation_backbone}
\resizebox{1\linewidth}{!}{
\begin{tabular}{p{3cm}<{\raggedright}|p{1cm}<{\centering}p{1cm}<{\centering}p{1cm}<{\centering}p{1cm}<{\centering}}
\toprule
Backbone & AP & AP$_\text{r}$ & AP$_\text{c}$ & AP$_\text{f}$ \\
\midrule
YOLOv8 & 35.6 & 33.9 & 35.8 & 37.3 \\
DINOv3 & N/A & N/A & N/A & N/A \\
DINOv3 (frozen) & 37.2 & 35.9 & 37.2 & 37.6\\
\bottomrule
\end{tabular}}
\end{table}

\section{Visualization}

Figure~\ref{fig:EAMvis} illustrates the intuition behind the proposed EAH.
Instead of collapsing a text prompt into a single global embedding, EAH preserves all token-level representations before the EOT token and computes a spatial alignment map for each token.
These token-wise maps are then combined through a soft expectation mechanism, where each token contributes with a learned importance weight.
As a result, the EOT token remains the dominant alignment signal inherited from CLIP pre-training, while informative non-EOT tokens (e.g., \emph{knee}, \emph{high}, \emph{socks}) provide complementary fine-grained cues that refine the spatial structure of the alignment map.
Rather than suppressing the EOT token, Expectation Alignment enhances it with token-level semantic details, enabling fine-grained vision-language alignment without introducing hard token selection or additional supervision.

We present qualitative results of ExpAlign in Figure \ref{fig:visual}. Subfigures \ref{fig:visuala} and \ref{fig:visualb} present correctly detected objects, where all corresponding prompt phrases are completely absent from the training data. This clearly demonstrates the robust zero-shot generalization capability of ExpAlign. Furthermore, the results in subfigures \ref{fig:visualc} and \ref{fig:visuald} reveal that the model exhibits a non-trivial level of referring expression comprehension (REC) ability, successfully grounding complex and novel expressions even in unseen scenarios. Notably, ExpAlign delivers exceptionally high-quality segmentation masks, with particularly impressive performance at object boundaries and under partial occlusion, highlighting its strong capability in precise instance delineation. See more examples in Appendix~\ref{app:vis_qualitative}.

\section{Conclusion}
In this paper, we presented \textbf{ExpAlign}, an expectation-guided vision-language alignment framework for open-vocabulary grounding under weak and ambiguous supervision. By introducing the Expectation Alignment Head (EAH), our method aggregates token-level vision-language similarities through a principled expectation mechanism, enabling implicit token selection and soft region alignment without relying on explicit instance-level annotations. Furthermore, we proposed a multi-scale consistency regularization strategy, including a Top-K multi-positive contrastive objective and a Geometry-Aware Consistency Objective, to jointly enhance semantic discriminability and spatial coherence of alignment maps during training. Extensive experiments on open-vocabulary detection and instance segmentation benchmarks demonstrate that ExpAlign consistently improves performance, particularly on long-tail categories and zero-shot segmentation quality, while remaining lightweight and fully compatible with standard detection and segmentation pipelines. We believe this work offers a practical and theoretically grounded step toward more expressive and robust vision-language alignment for open-world visual understanding.


\bibliography{ref}

@inproceedings{ilse2018attention,
  title={Attention-based deep multiple instance learning},
  author={Ilse, Maximilian and Tomczak, Jakub and Welling, Max},
  booktitle={International conference on machine learning},
  pages={2127--2136},
  year={2018},
  organization={PMLR}
}

@inproceedings{radford2021learning,
  title={Learning transferable visual models from natural language supervision},
  author={Radford, Alec and Kim, Jong Wook and Hallacy, Chris and Ramesh, Aditya and Goh, Gabriel and Agarwal, Sandhini and Sastry, Girish and Askell, Amanda and Mishkin, Pamela and Clark, Jack and others},
  booktitle={International conference on machine learning},
  pages={8748--8763},
  year={2021},
  organization={PmLR}
}

@inproceedings{li2023blip,
  title={Blip-2: Bootstrapping language-image pre-training with frozen image encoders and large language models},
  author={Li, Junnan and Li, Dongxu and Savarese, Silvio and Hoi, Steven},
  booktitle={International conference on machine learning},
  pages={19730--19742},
  year={2023},
  organization={PMLR}
}

@article{simeoni2025dinov3,
  title={Dinov3},
  author={Sim{\'e}oni, Oriane and Vo, Huy V and Seitzer, Maximilian and Baldassarre, Federico and Oquab, Maxime and Jose, Cijo and Khalidov, Vasil and Szafraniec, Marc and Yi, Seungeun and Ramamonjisoa, Micha{\"e}l and others},
  journal={arXiv preprint arXiv:2508.10104},
  year={2025}
}

@inproceedings{jia2021scaling,
  title={Scaling up visual and vision-language representation learning with noisy text supervision},
  author={Jia, Chao and Yang, Yinfei and Xia, Ye and Chen, Yi-Ting and Parekh, Zarana and Pham, Hieu and Le, Quoc and Sung, Yun-Hsuan and Li, Zhen and Duerig, Tom},
  booktitle={International conference on machine learning},
  pages={4904--4916},
  year={2021},
  organization={PMLR}
}

@article{oord2018representation,
  title={Representation learning with contrastive predictive coding},
  author={Oord, Aaron van den and Li, Yazhe and Vinyals, Oriol},
  journal={arXiv preprint arXiv:1807.03748},
  year={2018}
}

@article{ren2024grounding,
  title={Grounding dino 1.5: Advance the" edge" of open-set object detection},
  author={Ren, Tianhe and Jiang, Qing and Liu, Shilong and Zeng, Zhaoyang and Liu, Wenlong and Gao, Han and Huang, Hongjie and Ma, Zhengyu and Jiang, Xiaoke and Chen, Yihao and others},
  journal={arXiv preprint arXiv:2405.10300},
  year={2024}
}

@article{wang2025yoloe,
  title={Yoloe: Real-time seeing anything},
  author={Wang, Ao and Liu, Lihao and Chen, Hui and Lin, Zijia and Han, Jungong and Ding, Guiguang},
  journal={arXiv preprint arXiv:2503.07465},
  year={2025}
}

@inproceedings{cheng2024yolo,
  title={Yolo-world: Real-time open-vocabulary object detection},
  author={Cheng, Tianheng and Song, Lin and Ge, Yixiao and Liu, Wenyu and Wang, Xinggang and Shan, Ying},
  booktitle={Proceedings of the IEEE/CVF conference on computer vision and pattern recognition},
  pages={16901--16911},
  year={2024}
}

@inproceedings{liu2024grounding,
  title={Grounding dino: Marrying dino with grounded pre-training for open-set object detection},
  author={Liu, Shilong and Zeng, Zhaoyang and Ren, Tianhe and Li, Feng and Zhang, Hao and Yang, Jie and Jiang, Qing and Li, Chunyuan and Yang, Jianwei and Su, Hang and others},
  booktitle={European conference on computer vision},
  pages={38--55},
  year={2024},
  organization={Springer}
}

@inproceedings{zhou2022detecting,
  title={Detecting twenty-thousand classes using image-level supervision},
  author={Zhou, Xingyi and Girdhar, Rohit and Joulin, Armand and Kr{\"a}henb{\"u}hl, Philipp and Misra, Ishan},
  booktitle={European conference on computer vision},
  pages={350--368},
  year={2022},
  organization={Springer}
}

@inproceedings{li2022grounded,
  title={Grounded language-image pre-training},
  author={Li, Liunian Harold and Zhang, Pengchuan and Zhang, Haotian and Yang, Jianwei and Li, Chunyuan and Zhong, Yiwu and Wang, Lijuan and Yuan, Lu and Zhang, Lei and Hwang, Jenq-Neng and others},
  booktitle={Proceedings of the IEEE/CVF conference on computer vision and pattern recognition},
  pages={10965--10975},
  year={2022}
}

@inproceedings{kamath2021mdetr,
  title={Mdetr-modulated detection for end-to-end multi-modal understanding},
  author={Kamath, Aishwarya and Singh, Mannat and LeCun, Yann and Synnaeve, Gabriel and Misra, Ishan and Carion, Nicolas},
  booktitle={Proceedings of the IEEE/CVF international conference on computer vision},
  pages={1780--1790},
  year={2021}
}

@article{shao2024deepseekmath,
  title={Deepseekmath: Pushing the limits of mathematical reasoning in open language models},
  author={Shao, Zhihong and Wang, Peiyi and Zhu, Qihao and Xu, Runxin and Song, Junxiao and Bi, Xiao and Zhang, Haowei and Zhang, Mingchuan and Li, YK and Wu, Yang and others},
  journal={arXiv preprint arXiv:2402.03300},
  year={2024}
}

@inproceedings{shao2019objects365,
  title={Objects365: A large-scale, high-quality dataset for object detection},
  author={Shao, Shuai and Li, Zeming and Zhang, Tianyuan and Peng, Chao and Yu, Gang and Zhang, Xiangyu and Li, Jing and Sun, Jian},
  booktitle={Proceedings of the IEEE/CVF international conference on computer vision},
  pages={8430--8439},
  year={2019}
}

@inproceedings{yu2016modeling,
  title={Modeling context in referring expressions},
  author={Yu, Licheng and Poirson, Patrick and Yang, Shan and Berg, Alexander C and Berg, Tamara L},
  booktitle={European conference on computer vision},
  pages={69--85},
  year={2016},
  organization={Springer}
}

@inproceedings{plummer2015flickr30k,
  title={Flickr30k entities: Collecting region-to-phrase correspondences for richer image-to-sentence models},
  author={Plummer, Bryan A and Wang, Liwei and Cervantes, Chris M and Caicedo, Juan C and Hockenmaier, Julia and Lazebnik, Svetlana},
  booktitle={Proceedings of the IEEE international conference on computer vision},
  pages={2641--2649},
  year={2015}
}

@article{hudson2019gqa,
  title={Gqa: a new dataset for compositional question answering over real-world images},
  author={Hudson, Drew A and Manning, Christopher D},
  journal={arXiv preprint arXiv:1902.09506},
  volume={3},
  number={8},
  pages={1},
  year={2019}
}

@inproceedings{cai2022x,
  title={X-detr: A versatile architecture for instance-wise vision-language tasks},
  author={Cai, Zhaowei and Kwon, Gukyeong and Ravichandran, Avinash and Bas, Erhan and Tu, Zhuowen and Bhotika, Rahul and Soatto, Stefano},
  booktitle={European Conference on Computer Vision},
  pages={290--308},
  year={2022},
  organization={Springer}
}

@article{fu2025wedetect,
  title={WeDetect: Fast Open-Vocabulary Object Detection as Retrieval},
  author={Fu, Shenghao and Su, Yukun and Rao, Fengyun and Lyu, Jing and Xie, Xiaohua and Zheng, Wei-Shi},
  journal={arXiv preprint arXiv:2512.12309},
  year={2025}
}

@inproceedings{gupta2019lvis,
  title={Lvis: A dataset for large vocabulary instance segmentation},
  author={Gupta, Agrim and Dollar, Piotr and Girshick, Ross},
  booktitle={Proceedings of the IEEE/CVF conference on computer vision and pattern recognition},
  pages={5356--5364},
  year={2019}
}

@article{kang2025clip,
  title={Is CLIP ideal? No. Can we fix it? Yes!},
  author={Kang, Raphi and Song, Yue and Gkioxari, Georgia and Perona, Pietro},
  journal={arXiv preprint arXiv:2503.08723},
  year={2025}
}

@article{zhang2022glipv2,
  title={Glipv2: Unifying localization and vision-language understanding},
  author={Zhang, Haotian and Zhang, Pengchuan and Hu, Xiaowei and Chen, Yen-Chun and Li, Liunian and Dai, Xiyang and Wang, Lijuan and Yuan, Lu and Hwang, Jenq-Neng and Gao, Jianfeng},
  journal={Advances in Neural Information Processing Systems},
  volume={35},
  pages={36067--36080},
  year={2022}
}

@article{zeng2021multi,
  title={Multi-grained vision language pre-training: Aligning texts with visual concepts},
  author={Zeng, Yan and Zhang, Xinsong and Li, Hang},
  journal={arXiv preprint arXiv:2111.08276},
  year={2021}
}

@article{kang2025empower,
  title={Empower Words: DualGround for Structured Phrase and Sentence-Level Temporal Grounding},
  author={Kang, Minseok and Lee, Minhyeok and Kim, Minjung and Kim, Donghyeong and Lee, Sangyoun},
  journal={arXiv preprint arXiv:2510.20244},
  year={2025}
}

@inproceedings{liang2021exploring,
  title={Exploring geometry-aware contrast and clustering harmonization for self-supervised 3d object detection},
  author={Liang, Hanxue and Jiang, Chenhan and Feng, Dapeng and Chen, Xin and Xu, Hang and Liang, Xiaodan and Zhang, Wei and Li, Zhenguo and Van Gool, Luc},
  booktitle={Proceedings of the IEEE/CVF International Conference on Computer Vision},
  pages={3293--3302},
  year={2021}
}

@inproceedings{achal2022evaluating,
  title={Evaluating Large-Vocabulary Object Detectors: The Devil is in the Details},
  author={Achal Dave, PD and Ramanan, D and Kirillov, A and Girshick, R},
  year={2022},
  organization={CVPR}
}

@article{yao2022detclip,
  title={Detclip: Dictionary-enriched visual-concept paralleled pre-training for open-world detection},
  author={Yao, Lewei and Han, Jianhua and Wen, Youpeng and Liang, Xiaodan and Xu, Dan and Zhang, Wei and Li, Zhenguo and Xu, Chunjing and Xu, Hang},
  journal={Advances in Neural Information Processing Systems},
  volume={35},
  pages={9125--9138},
  year={2022}
}

@article{shazeer2020glu,
  title={Glu variants improve transformer},
  author={Shazeer, Noam},
  journal={arXiv preprint arXiv:2002.05202},
  year={2020}
}

@inproceedings{varghese2024yolov8,
  title={Yolov8: A novel object detection algorithm with enhanced performance and robustness},
  author={Varghese, Rejin and Sambath, M},
  booktitle={2024 International conference on advances in data engineering and intelligent computing systems (ADICS)},
  pages={1--6},
  year={2024},
  organization={IEEE}
}

@inproceedings{lin2014microsoft,
  title={Microsoft coco: Common objects in context},
  author={Lin, Tsung-Yi and Maire, Michael and Belongie, Serge and Hays, James and Perona, Pietro and Ramanan, Deva and Doll{\'a}r, Piotr and Zitnick, C Lawrence},
  booktitle={European conference on computer vision},
  pages={740--755},
  year={2014},
  organization={Springer}
}

@article{ravi2024sam,
  title={Sam 2: Segment anything in images and videos},
  author={Ravi, Nikhila and Gabeur, Valentin and Hu, Yuan-Ting and Hu, Ronghang and Ryali, Chaitanya and Ma, Tengyu and Khedr, Haitham and R{\"a}dle, Roman and Rolland, Chloe and Gustafson, Laura and others},
  journal={arXiv preprint arXiv:2408.00714},
  year={2024}
}

@inproceedings{zhong2022regionclip,
  title={Regionclip: Region-based language-image pretraining},
  author={Zhong, Yiwu and Yang, Jianwei and Zhang, Pengchuan and Li, Chunyuan and Codella, Noel and Li, Liunian Harold and Zhou, Luowei and Dai, Xiyang and Yuan, Lu and Li, Yin and others},
  booktitle={Proceedings of the IEEE/CVF conference on computer vision and pattern recognition},
  pages={16793--16803},
  year={2022}
}

@article{singh2025varp,
  title={VARP: Reinforcement Learning from Vision-Language Model Feedback with Agent Regularized Preferences},
  author={Singh, Anukriti and Bhaskar, Amisha and Yu, Peihong and Chakraborty, Souradip and Dasyam, Ruthwik and Bedi, Amrit and Tokekar, Pratap},
  journal={arXiv preprint arXiv:2503.13817},
  year={2025}
}

@article{zhengprll,
  title={PRLL: Policy Regularization and Reward Shaping Assisted by Large Language Models},
  author={Zheng, Qianxia and Luo, Xiangfeng and Wang, Tao}
}

@article{zhai2024fine,
  title={Fine-tuning large vision-language models as decision-making agents via reinforcement learning},
  author={Zhai, Simon and Bai, Hao and Lin, Zipeng and Pan, Jiayi and Tong, Peter and Zhou, Yifei and Suhr, Alane and Xie, Saining and LeCun, Yann and Ma, Yi and others},
  journal={Advances in neural information processing systems},
  volume={37},
  pages={110935--110971},
  year={2024}
}

@article{huang2025vlm,
  title={Vlm-rl: A unified vision language models and reinforcement learning framework for safe autonomous driving},
  author={Huang, Zilin and Sheng, Zihao and Qu, Yansong and You, Junwei and Chen, Sikai},
  journal={Transportation Research Part C: Emerging Technologies},
  volume={180},
  pages={105321},
  year={2025},
  publisher={Elsevier}
}

@inproceedings{zhang2023simple,
  title={A simple framework for open-vocabulary segmentation and detection},
  author={Zhang, Hao and Li, Feng and Zou, Xueyan and Liu, Shilong and Li, Chunyuan and Yang, Jianwei and Zhang, Lei},
  booktitle={Proceedings of the IEEE/CVF International Conference on Computer Vision},
  pages={1020--1031},
  year={2023}
}

@article{wang2024ovlw,
  title={OVLW-DETR: Open-Vocabulary Light-Weighted Detection Transformer},
  author={Wang, Yu and Su, Xiangbo and Chen, Qiang and Zhang, Xinyu and Xi, Teng and Yao, Kun and Ding, Errui and Zhang, Gang and Wang, Jingdong},
  journal={arXiv preprint arXiv:2407.10655},
  year={2024}
}

@article{zhao2024omdet,
  title={OmDet: Large-scale vision-language multi-dataset pre-training with multimodal detection network},
  author={Zhao, Tiancheng and Liu, Peng and Lee, Kyusong},
  journal={IET Computer Vision},
  year={2024},
  publisher={Wiley Online Library}
}
\bibliographystyle{icml2026}

\newpage
\appendix
\onecolumn
\section{Connection to Multiple Instance Learning}
\label{app:mil_supp}

Although ExpAlign is presented as a vision-language alignment module rather than a canonical MIL algorithm, it admits an exact interpretation and equivalence to attention-based soft pooling in the MIL framework. Below we give a concise mapping and a proof sketch that justifies the claim in Section~\ref{sec:mil}.

\textbf{Notation.} Fix a textual prompt $p$ and a feature scale $s$. Let $\Omega=\{1,\dots,N\}$ index spatial locations in the feature map ($N=H_sW_s$), and let $L$ denote the number of valid text tokens. For each spatial location $i\in\Omega$ and token $l\in\{1,\dots,L\}$, define the token--patch affinity
\[
S(i,l)=\langle F_s(x,y),\,T_{b,p}(l)\rangle,
\]
where $i$ is the flattened index of $(x,y)$. The spatially averaged response of token $l$ is
\[
\bar S(l)=\frac{1}{N}\sum_{i\in\Omega} S(i,l),
\]
and the token posterior is given by
\[
\pi(l)=\frac{\exp(\bar S(l)/\tau_t)}{\sum_{l'=1}^L \exp(\bar S(l')/\tau_t)}.
\]
The EAM assigns to each spatial location the score
\[
\widetilde S(i)=\sum_{l=1}^L \pi(l)\,S(i,l).
\]

\textbf{Reformulation as instance-wise linear pooling.}
For each instance $i$, define the token-affinity vector
\[
v_i=(S(i,1),\dots,S(i,L))^\top \in \mathbb{R}^L,
\]
and collect the token posteriors into $\pi\in\mathbb{R}^L$. With this notation, the EAM score can be written compactly as
\[
\widetilde S(i)=\pi^\top v_i,
\]
which shows that each instance score is obtained by applying the same linear functional to its token-affinity vector.

\textbf{MIL interpretation and bag-level aggregation.}
From the MIL perspective, the prompt $p$ defines a bag whose instances are the unordered set $\{v_i\}_{i\in\Omega}$. The mapping $v_i\mapsto\widetilde S(i)$ is permutation equivariant, and the subsequent aggregation used by ExpAlign,
\[
\ell=\frac{1}{|\mathrm{TopK}(\widetilde S)|}\sum_{i\in\mathrm{TopK}(\widetilde S)} \widetilde S(i),
\]
is permutation invariant. Such a construction satisfies the defining requirement of MIL pooling operators and corresponds to a Top-$K$ variant of attention-based soft pooling, where discriminative instances dominate the bag-level response.

\textbf{Equivalence to attention-based MIL pooling.}
Attention-based MIL methods \citep{ilse2018attention} compute a scalar score for each instance via an attention mechanism and aggregate these scores using a permutation-invariant operator. In ExpAlign, attention is factorized into a token-level posterior $\pi$, shared across instances, followed by instance-level pooling over $\widetilde S(i)$. Algebraically, both formulations reduce to computing instance scores $g(v_i)$ and applying a soft or Top-$K$ aggregation over instances. The difference lies only in how attention weights are parameterized, not in the form of the pooling operator.

\textbf{Permutation invariance and expressiveness.}
Because $\pi$ depends only on the set $\{v_i\}$ through the averaged statistics $\{\bar S(l)\}$, the overall operator from $\{v_i\}$ to $\ell$ is permutation invariant. Moreover, by adjusting the temperature $\tau_t$ and the Top-$K$ ratio, the pooling behavior interpolates between mean, max, and soft-attention pooling, matching the expressive family of attention-based MIL operators.

\textbf{Discussion.}
This equivalence clarifies that ExpAlign performs a principled MIL-style soft selection over instances while allowing uncertainty at both the token and spatial levels. The auxiliary multi-positive InfoNCE loss and the Geometry-Aware Consistency Objective can thus be viewed as bag-level discriminative losses and intra-bag energy shaping, respectively, consistent with standard MIL training principles.

\textbf{Remarks.} For completeness, the variational derivation in Appendix~\ref{app:variational} further shows that the geometry-aware consistency term yields a Gibbs reweighting of prompt--patch probabilities under a Lagrangian-constrained free-energy, which reshapes intra-instance mass without requiring explicit instance labels.

\section{Variational Derivation of Gibbs Reweighting in Energy-Based Consistency Regularization}\label{app:variational}

We consider a finite collection of prompt--patch pairs indexed by $(p,i)$,
with $p\in\{1,\dots,P\}$ and $i\in\Omega$ ($|\Omega|$ finite).
For compactness we sometimes write a generic index $\alpha$ to denote a pair $(p,i)$.

\medskip
\begin{assumption}[Energy field]
    There is a real-valued alignment score field $\tilde S_{p}(i)\in\mathbb R$.
    We define the associated \emph{energy} by
    \begin{equation}\label{eq:energy_def}
        E(p,i) \;=\; -\tilde S_{p}(i).
    \end{equation}
    We assume $E(p,i)$ is uniformly bounded on the finite domain.
\end{assumption}

\begin{assumption}[Instance-wise geometry score]
    For each image $b$ and prompt $p$ we are given a bounded geometry score $A_{b,p}(i)\in\mathbb R$ defined for $i\in\Omega$ such that:
    \begin{enumerate}
      \item $A_{b,p}$ depends only on intra-instance relative statistics (e.g. mean and standard deviation computed over the ground-truth mask $\mathcal M_{b,p}$), and is therefore invariant to adding a constant to $\tilde S$ (affine invariance in the additive sense) and to monotone affine rescaling when appropriately adjusting normalization;
      \item $A_{b,p}$ is bounded and (locally) Lipschitz in $\tilde S$ (so gradient bounds exist and empirical gradients are well-defined).
    \end{enumerate}
\end{assumption}

\begin{assumption}[Regularization parameters]
    Let $\tau>0$ be the temperature (entropy weight) and $\lambda\in\mathbb R$ be the geometry weight (we will take $\lambda\ge0$ in most discussion).
\end{assumption}

\medskip

We denote by $\mathcal P$ the probability simplex over all prompt--patch indices:
\[
\mathcal P = \Big\{ \mathbb Q:\ \mathbb Q(p,i)\ge0,\ \sum_{p,i}\mathbb Q(p,i)=1 \Big\}.
\]
Let $\mathbb U$ denote the uniform distribution on the finite set of pairs $(p,i)$,
i.e. $\mathbb U(p,i)=1/(P|\Omega|)$.

\begin{theorem}[Variational optimality and induced Gibbs form]\label{thm:vari-gibbs-form}
    Under Assumptions B1--B3, consider the variational free-energy functional
    \begin{equation}\label{eq:free_energy}
    \mathcal F[\mathbb Q]
    \;=\;
    \mathbb E_{\mathbb Q}\big[ E(p,i)\big]
    \;-\;
    \lambda\,\mathbb E_{\mathbb Q}\big[ A_{b,p}(i)\big]
    \;+\;
    \tau\,\mathrm{KL}\big(\mathbb Q \;\|\; \mathbb U\big),
    \qquad
    \mathbb Q\in\mathcal P.
    \end{equation}
    Then:
    \begin{enumerate}
      \item The functional $\mathcal F$ is strictly convex on $\mathcal P$ and admits a unique minimizer $\mathbb Q^\star\in\mathcal P$.
      \item The minimizer has the explicit Gibbs (exponential-family) form
      \begin{equation}\label{eq:gibbs_solution}
      \mathbb Q^\star(p,i)
      \;=\;
      \frac{\exp\!\big( -\frac{1}{\tau}\big(E(p,i)-\lambda A_{b,p}(i)\big)\big)}
      {\sum_{p',i'} \exp\!\big( -\frac{1}{\tau}\big(E(p',i')-\lambda A_{b',p'}(i')\big)\big)}.
      \end{equation}
      \item Equivalently, substituting $E(p,i)=-\tilde S_p(i)$, the optimal distribution can be written
      \[
      \mathbb Q^\star(p,i)
      \;=\;
      \frac{\exp\!\big( \tfrac{1}{\tau}\big(\tilde S_p(i)+\lambda A_{b,p}(i)\big)\big)}
      {\sum_{p',i'} \exp\!\big( \tfrac{1}{\tau}\big(\tilde S_{p'}(i')+\lambda A_{b',p'}(i')\big)\big)}.
      \]
      \item Moreover, minimizing the cross-entropy (or KL divergence) from an empirical target distribution
      $Q_{\mathrm{target}}(p,i)\propto \mathbb{1}_{i\in\mathcal M_{b,p}}\,w_{b,p}(i)$
      to $\mathbb Q^\star$ yields the geometry-aware loss of the form
      \[
      \mathcal L_{\mathrm{geo}}
      \;=\;
      -\sum_{b,p}\sum_{i\in\mathcal M_{b,p}} \tilde c_{b,p,i}\,\log \mathbb Q^\star(p,i),
      \]
      which is equivalent up to normalization constants to the loss reported in the main text.
    \end{enumerate}
\end{theorem}

\textbf{Remarks.}
\begin{itemize}
  \item The KL term provides strict convexity and enforces positive entropy, preventing collapse to a point mass; the linear terms (expectation of $E$ and $A$) are affine in $\mathbb Q$ and therefore preserve convexity.
  \item The parameter $\tau$ controls the trade-off between fidelity to energy $E$ and entropy (stability vs. selectivity); $\lambda$ controls the strength of instance-local geometric shaping.
  \item Additive shifts of $E$ (i.e. $E\mapsto E+c$) do not change $\mathbb Q^\star$; multiplicative rescaling of $E$ can be absorbed into $\tau$ (i.e. $aE/\tau=(E)/( \tau/a)$).
\end{itemize}

\begin{proof}
    We supply a complete and explicit derivation in several carefully enumerated steps.

    Work on the finite index set $\mathcal I=\{(p,i)\}$. Any $\mathbb Q\in\mathcal P$ can be represented as a vector $\mathbb Q\in\mathbb R^{|\mathcal I|}$ with nonnegative entries summing to one.
    On this finite dimensional simplex all functions below are well-defined and differentiable on the interior.
    
    \textbf{Convexity and existence/uniqueness.}
    
    Observe that $\mathcal F[\mathbb Q]$ in \eqref{eq:free_energy} can be written as
    \[
    \mathcal F[\mathbb Q]
    =
    \sum_{(p,i)\in\mathcal I} \mathbb Q(p,i)\big( E(p,i)-\lambda A_{b,p}(i)\big)
    +
    \tau\sum_{(p,i)} \mathbb Q(p,i)\log\frac{\mathbb Q(p,i)}{\mathbb U(p,i)}.
    \]
    The first term is linear in $\mathbb Q$; the second term is $\tau$ times the relative entropy (KL), which is strictly convex in $\mathbb Q$ on the interior of the simplex. Hence $\mathcal F$ is strictly convex. Because $\mathcal P$ is compact and $\mathcal F$ is continuous, a unique minimizer exists.
    
    \textbf{First-order optimality (variational derivative).}
    
    To find the minimizer, form the Lagrangian for the constrained minimization (constraint: $\sum_{(p,i)}\mathbb Q(p,i)=1$):
    \[
    \mathcal L(\mathbb Q,\eta)
    =
    \sum_{(p,i)} \mathbb Q(p,i)\big( E(p,i)-\lambda A_{b,p}(i)\big)
    +
    \tau\sum_{(p,i)} \mathbb Q(p,i)\log\frac{\mathbb Q(p,i)}{\mathbb U(p,i)}
    + \eta\big( \sum_{(p,i)}\mathbb Q(p,i) - 1\big),
    \]
    where $\eta\in\mathbb R$ is the Lagrange multiplier enforcing normalization.
    
    Take partial derivative with respect to $\mathbb Q(\bar p,\bar i)$ (interior point) and set to zero:
    \begin{align*}
    0 \;=\; \frac{\partial\mathcal L}{\partial \mathbb Q(\bar p,\bar i)}
    &=
    E(\bar p,\bar i)-\lambda A_{b,\bar p}(\bar i)
    + \tau\Big( \log\frac{\mathbb Q(\bar p,\bar i)}{\mathbb U(\bar p,\bar i)} + 1\Big)
    + \eta.
    \end{align*}
    Rearrange to isolate the log term:
    \[
    \log\frac{\mathbb Q(\bar p,\bar i)}{\mathbb U(\bar p,\bar i)}
    =
    -\frac{1}{\tau}\big(E(\bar p,\bar i)-\lambda A_{b,\bar p}(\bar i)\big)
    -1 - \frac{\eta}{\tau}.
    \]
    Exponentiating both sides yields
    \[
    \mathbb Q(\bar p,\bar i)
    =
    \mathbb U(\bar p,\bar i)
    \exp\!\Big( -\frac{1}{\tau}\big(E(\bar p,\bar i)-\lambda A_{b,\bar p}(\bar i)\big)\Big)
    \cdot \exp\!\Big(-1-\frac{\eta}{\tau}\Big).
    \]
    Since $\exp(-1-\eta/\tau)$ is a global scalar independent of $(\bar p,\bar i)$, normalization enforces that this scalar equals the reciprocal of the partition sum. Using the explicit form of $\mathbb U(p,i)$ (uniform), we obtain the normalized Gibbs form
    \[
    \mathbb Q^\star(p,i)
    =
    \frac{\exp\!\big( -\frac{1}{\tau}\big(E(p,i)-\lambda A_{b,p}(i)\big)\big)}
    {\sum_{p',i'} \exp\!\big( -\frac{1}{\tau}\big(E(p',i')-\lambda A_{b',p'}(i')\big)\big)}.
    \]
    This completes the derivation of \eqref{eq:gibbs_solution} and establishes both necessity and sufficiency of this form for optimality (sufficiency follows from strict convexity).
    
    \textbf{Substitution $E=-\tilde S$ and alternative form.}
    Using $E(p,i)=-\tilde S_{p}(i)$, rewrite \eqref{eq:gibbs_solution} as
    \[
    \mathbb Q^\star(p,i)
    =
    \frac{\exp\!\big( \tfrac{1}{\tau}\big(\tilde S_p(i)+\lambda A_{b,p}(i)\big)\big)}
    {\sum_{p',i'} \exp\!\big( \tfrac{1}{\tau}\big(\tilde S_{p'}(i')+\lambda A_{b',p'}(i')\big)\big)}.
    \]
    This shows that the geometry term $A_{b,p}(i)$ directly enters the logits of the Gibbs distribution and hence modifies the model posterior in a multiplicative exponential manner.
    
    \textbf{Equivalence to cross-entropy style training loss.}
    
    Suppose we define a target empirical distribution on $(p,i)$ for training,
    \[
    Q_{\mathrm{target}}(p,i)
    \;=\;
    \frac{\mathbb{1}\{i\in\mathcal M_{b,p}\}\, w_{b,p}(i)}
    {\sum_{p',i'} \mathbb{1}\{i'\in\mathcal M_{b',p'}\}\, w_{b',p'}(i')},
    \]
    where $w_{b,p}(i)$ is a nonnegative weight (e.g. $w_{b,p}(i)=A_{b,p}(i)$ or another monotone transform). The standard cross-entropy (expected negative log-likelihood) of this target under model $\mathbb Q$ is
    \[
    \mathrm{CE}(Q_{\mathrm{target}}\| \mathbb Q)
    =
    -\sum_{p,i} Q_{\mathrm{target}}(p,i)\,\log \mathbb Q(p,i).
    \]
    Minimizing this CE over model parameters (i.e. making $\mathbb Q$ approximate $Q_{\mathrm{target}}$) is equivalent to minimizing $\mathrm{KL}(Q_{\mathrm{target}}\|\mathbb Q)$ up to an additive entropy constant $H(Q_{\mathrm{target}})$ independent of model. When the model is constrained to the Gibbs family as in \eqref{eq:gibbs_solution}, minimizing CE corresponds to adjusting free-energy parameters (and indirectly logits $\tilde S$ and geometry weight $\lambda$) so that $\mathbb Q^\star$ matches $Q_{\mathrm{target}}$. Thus the training objective
    \[
    \mathcal L_{\mathrm{geo}}
    =
    -\sum_{b,p}\sum_{i\in\mathcal M_{b,p}} \tilde c_{b,p,i}\,\log \mathbb Q^\star(p,i)
    \]
    is precisely the empirical counterpart of the variational optimization \eqref{eq:free_energy} when choosing $Q_{\mathrm{target}}$ proportional to instance-local geometry weights.
\end{proof}

    \textbf{Additional properties (invariance and non-collapse).}
    
    \begin{itemize}
      \item \emph{Additive invariance.} If $E\mapsto E+c$ (for constant $c$), then the numerator of \eqref{eq:gibbs_solution} acquires factor $\exp(-c/\tau)$ independent of $(p,i)$ and cancels with the denominator; hence $\mathbb Q^\star$ is invariant to additive shifts of energy, equivalently to adding constants to $\tilde S$.
      \item \emph{Scaling and temperature.} If $E$ is multiplied by positive scalar $a>0$, then
      \[
      \exp\!\big(-\tfrac{1}{\tau} a E\big) = \exp\!\big(-\tfrac{1}{\tau/a} E\big),
      \]
      so multiplicative rescaling of $E$ can be absorbed into a reparametrization of $\tau$ (temperature).
      \item \emph{Non-collapse (positive entropy).} Because $\tau>0$ and the KL term penalizes zero entropy, the minimizer $\mathbb Q^\star$ has strictly positive entropy (unless the energy differences are arbitrarily large compared to $\tau$). In particular $\mathbb Q^\star$ is not a point mass unless the limit $\tau\downarrow0$ is taken.
      \item \emph{Instance-local perturbation.} If $A_{b,p}(i)$ is supported only on indices $i$ belonging to ground-truth instance $\mathcal M_{b,p}$, then the additive perturbation $\lambda A_{b,p}(i)$ only affects relative probabilities within that instance; it does not change ordering of energies between different instances except insofar as their partition sums change, and thus constitutes a \emph{conditional} (instance-wise) energy shaping.
    \end{itemize}
    
    \textbf{Limit cases and interpretation.}
    \begin{itemize}
      \item As $\tau\to\infty$, the KL penalty dominates and $\mathbb Q^\star$ tends to the uniform distribution $\mathbb U$ (max-entropy limit).
      \item As $\tau\to0^{+}$, $\mathbb Q^\star$ concentrates on the minimizers of $E(p,i)-\lambda A_{b,p}(i)$ (hard selection / argmax).
      \item As $\lambda\to0$, one recovers the standard Gibbs posterior based solely on $E$ (i.e. the semantic-only reweighting).
      \item Intermediate $(\tau,\lambda)$ trade off stability (entropy), semantic fidelity (alignment to $\tilde S$), and geometric consistency.
    \end{itemize}
    
    This completes the derivation and justification of the Gibbs reweighting and conditional energy shaping regularizers used in the main text.

\section{Comparative Evaluation of ExpAlign, Grounding DINO, and GLIP on Diverse Real-World Datasets in ODinW}
In our comparison of ExpAlign, Grounding DINO, and GLIP across the diverse real-world datasets in the ODinW benchmark, as presented in Table \ref{tab:comp_odinw}, ExpAlign demonstrates competitive overall performance with a slightly higher average score (22.4) than Grounding DINO (22.3) and a clear advantage over GLIP (19.6), while also showing strong gains on several challenging and domain-specific subsets.

For instance, ExpAlign substantially outperforms both baselines on datasets involving uncommon or underrepresented scenarios, such as MountainDewCommercial (45.46 vs. 25.46 for Grounding DINO and 21.60 for GLIP), ShellfishOpenImages* (42.63 vs. 29.56 and 25.90), MaskWearing (7.83 vs. 0.25 and 1.10), and PKLot\_640 (5.23 vs. 0.06 and 0.00). These improvements are likely attributable to ExpAlign's training on RefCOCO, which emphasizes referring expression comprehension and finer-grained grounding of objects in complex or natural-language contexts, helping the model better handle rare categories, occluded objects, or domain shifts not well covered in the O365 + GoldG + Cap4M pre-training corpus shared by Grounding DINO and GLIP.

Notably, on the ODinW-13 benchmark subset (marked with *), ExpAlign also achieves leading results in several cases (e.g., CottontailRabbits, EgoHands\_generic, pistols, VehiclesOpenImages), underscoring its enhanced generalization in high-quality, diverse open-world evaluation settings. These observations highlight the value of incorporating referring expression data during pre-training to boost robustness on out-of-distribution and long-tail categories in real-world object detection.

\begin{table*}[!htbp]
    \centering
    \begin{tabular}{l|ccc}
    \toprule
    Metric                   & GLIP-T & Grounding DINO T  & ExpAlign            \\ 
    \midrule
    Average Score            & 19.6   &  22.3  & \textbf{22.4}\\
    Median Score             & 5.1    &  \textbf{11.9}  & 10.4\\
    \midrule
    AerialMaritimeDrone\_large*      & \textbf{13.70 } &  10.30 &  11.15\\
    AerialMaritimeDrone\_tiled      & 12.60  &  17.50 &  \textbf{23.66}\\
    AmericanSignLanguageLetters     & \textbf{2.50}   &  0.78  & 2.34\\
    Aquarium*                 & 18.30  &  18.64 & \textbf{19.88}\\
    BCCD                     & 1.00   &  11.96 & \textbf{13.98}\\
    ChessPieces              & 10.00  &  \textbf{15.62} & 7.87\\
    CottontailRabbits*        & 69.70  &  67.61 & \textbf{81.34}\\
    DroneControl             & \textbf{5.10}   &  4.99  & 0.37\\
    EgoHands\_generic*        & 50.00  &  57.64 & \textbf{61.50}\\
    EgoHands\_specific       & 0.80   &  0.69  & \textbf{0.91}\\
    HardHatWorkers           & 3.00   &  \textbf{4.05}  & 3.30\\
    MaskWearing              & 1.10   &  0.25  & \textbf{7.83}\\
    MountainDewCommercial    & 21.60  &  25.46 & \textbf{45.46}\\
    NorthAmericaMushrooms*    & \textbf{75.10}  &  68.18 & 25.38\\
    OxfordPets\_by-breed     & 0.40   &  0.21  & \textbf{2.01}\\
    OxfordPets\_by-species   & 1.10   &  1.30  & \textbf{7.20}\\
    PKLot\_640               & 0.00   &  0.06  & \textbf{5.23}\\
    Packages*                 & \textbf{72.30}  &  60.53 & 71.65\\
    PascalVOC*                & 56.10  &  55.65 & \textbf{59.00}\\
    Raccoon*                  & 57.80  &  \textbf{60.07} & 46.29\\
    ShellfishOpenImages*      & 25.90  &  29.56 & \textbf{42.63}\\
    ThermalCheetah           & 2.70   &  \textbf{17.72} & 12.65\\
    UnoCards                 & 0.20   &  0.81  & \textbf{0.99}\\
    VehiclesOpenImages*       & 56.00  &  58.49 & \textbf{61.40}\\
    WildfireSmoke            & 2.30   &  \textbf{20.04} & 10.38\\
    boggleBoards             & 0.00   &  0.29  & \textbf{1.33}\\
    brackishUnderwater       & 3.70   &  1.47  & \textbf{3.70}\\
    dice                     & \textbf{1.10}   &  0.33  & 0.36\\
    openPoetryVision         & 0.00   &  0.05  & \textbf{0.16}\\
    pistols*                  & 49.80  &  66.99 & \textbf{75.69}\\
    plantdoc                 & 1.10   &  0.36  & \textbf{2.37}\\
    pothole*                  & 17.20  &  \textbf{25.21} & 7.30\\
    selfdrivingCar           & 8.00   &  \textbf{9.95}  & 8.21\\
    thermalDogsAndPeople*     & 43.70  &  \textbf{67.89} & 57.04\\
    websiteScreenshots       & 0.50   &  1.30  & \textbf{2.64}\\ 
    \bottomrule
    \end{tabular}
    \centering
    \caption{Comparison of ExpAlign, Grounding DINO, and GLIP on the ODinW benchmark.
Grounding DINO and GLIP are trained on Objects365, GoldG, and Cap4M using Swin-Tiny backbones.ExpAlign is trained on Objects365, GoldG, and RefCOCO using ConvNeXt-Tiny backbones. *denotes results belonging to the ODinW-13 benchmark.}
    \label{tab:comp_odinw}
\end{table*}

\section{GACO Pseudocode}

\begin{algorithm}[tb]
  \caption{Geometry-Aware Consistency Objective (GACO)}
  \label{alg:gaco}
  \begin{algorithmic}
    \STATE {\bfseries Input:} similarity map $sim \in [-1,1]$ of shape $[B, K, H, W]$, binary mask $M \in \{0,1\}$ of shape $[B, K, H, W]$, hyperparameters $\beta$, $adv\_clip$, $\epsilon$
    \STATE {\bfseries Output:} geometry consistency loss $\mathcal{L}_{\text{geo}}$

    \STATE Normalize $sim \gets sim / (|sim|_{\max} + \epsilon)$
    \STATE Compute logits $\gets sim.view(B, K \cdot H \cdot W)$
    \STATE Compute $\log p \gets \log\text{softmax}(\text{logits}, \dim=1)$

    \STATE Flatten $M$ to $M_{\text{flat}} \in [B, K \cdot H \cdot W]$
    \STATE Compute probability $prob_{\text{flat}} \gets \sigma(sim).view(B, K \cdot H \cdot W)$

    \STATE Initialize advantage-weighted loss $L_{\text{adv}} \gets 0$, denominator $denom \gets 0$

    \FOR{each batch $b = 0$ {\bfseries to} $B-1$}
        \STATE Find positive pixel indices $pos\_idx \gets$ where $M_{\text{flat}}[b] > 0.5$
        \IF{$pos\_idx$ is not empty}
            \STATE $R_{\text{pos}} \gets prob_{\text{flat}}[b, pos\_idx]$
            \STATE $\mu \gets \text{mean}(R_{\text{pos}})$, $\sigma \gets \text{std}(R_{\text{pos}}) + \epsilon$
            \STATE Advantage $A \gets (R_{\text{pos}} - \mu) / \sigma$
            \STATE Clamp $A \gets \text{clamp}(A, -adv\_clip, adv\_clip)$
            \STATE Accumulate $L_{\text{adv}} \gets L_{\text{adv}} - (A \cdot \log p[b, pos\_idx]).sum()$
            \STATE $denom \gets denom + |pos\_idx|$
        \ENDIF
    \ENDFOR

    \STATE $L_{\text{adv}} \gets L_{\text{adv}} / denom$ if $denom > 0$ else $0$

    \STATE $\mathcal{L}_{\text{geo}} \gets \beta \cdot L_{\text{adv}}$
  \end{algorithmic}
\end{algorithm}

This section provides the pseudocode of the Geometry-Aware Consistency Objective (GACO) used in all experiments. The objective operates on prompt-conditioned patch-level similarity maps and enforces consistency by reshaping the distribution of alignment scores within positive regions. Specifically, GACO treats the normalized similarity scores as an energy field over spatial locations and derives a Gibbs-style reweighting through a log-softmax normalization. Within each positive region, instance-level responses are standardized using region statistics, yielding an advantage signal that emphasizes relatively confident locations while preserving uncertainty. The final loss is computed as an advantage-weighted log-likelihood over masked locations, which corresponds to a constrained reweighting of spatial energies rather than explicit instance-level supervision. Algorithm~\ref{alg:gaco} summarizes the exact implementation used in our method.

\section{Hyper Paramerter Settings}
ExpAlign is trained using a two-stage protocol with frozen image and text encoders in both stages. Stage 1 focuses on semantic alignment with a moderate learning rate schedule and standard augmentations, while Stage 2 introduces geometry-aware consistency (GACO) and multi-positive contrastive weighting with reduced augmentation strength. Detailed hyperparameters are provided in Table~\ref{tab:hyper-parames}.

\begin{table*}[!htbp]
    \centering
    \begin{tabular}{l|cc}
    \toprule
    Hyper Paramerter         & Stage 1 Training & Stage 2 Training \\ 
    \midrule
    image size               & 640x640& 640x640  \\
    batch size               & 512   &  512  \\
    epochs                   & 30    &  20  \\
    warmup epochs            & 3     &  0   \\
    weight decay             & 0.025 &  0.025  \\
    initial learning rate    & 0.002 &  0.01  \\
    final learning rate fraction & 0.001 &  0.1  \\
    bias learning rate warmup & 0.0 &  0.0  \\
    momentum                 & 0.9 &  0.9  \\
    AMP training             & True &  True  \\
    freeze image encoder     & True &  True  \\
    freeze text encoder      & True &  True  \\
    multi-positive InfoNCE weight & 0.0 &  0.5  \\
    GACO weight              & 0.0 &  1.0  \\
    \midrule
    mosaic                   & 1.0   &  1.0  \\
    close mosaic             & 2     &  5  \\
    hsv h                    & 0.015 &  0.005   \\
    hsv s                    & 0.7   &  0.05  \\
    hsv v                    & 0.4   &  0.05  \\
    fliplr                   & 0.5   &  0.0  \\
    \bottomrule
    \end{tabular}
    \centering
    \caption{ExpAlign training hyper-parameters.}
    \label{tab:hyper-parames}
\end{table*}

\section{EAM Heatmap Visualizations for Negative Prompts}

Figure \ref{fig:eam-neg} presents the visualization of Explainable Attention Map (EAM) for negative sample prompt words on an example image of a girl in a sailor uniform. The left subfigure shows the original image, while the middle and right subfigures display EAM heatmaps for the positive prompt \textit{sailor uniform} and negative prompt \textit{black sailor uniform}, respectively.

As observed, for negative prompt, the EAM activations are more uniformly distributed across the background rather than concentrating on the foreground object (the uniform). This pattern suggests that the model suppresses the detection of negative prompts by diffusing attention, reducing false positives in irrelevant regions and enhancing overall robustness in prompt-guided tasks.

\begin{figure*}[!htbp]
    \centering
    \begin{subfigure}[t]{0.325\textwidth}
        \centering
        \includegraphics[width=\linewidth]{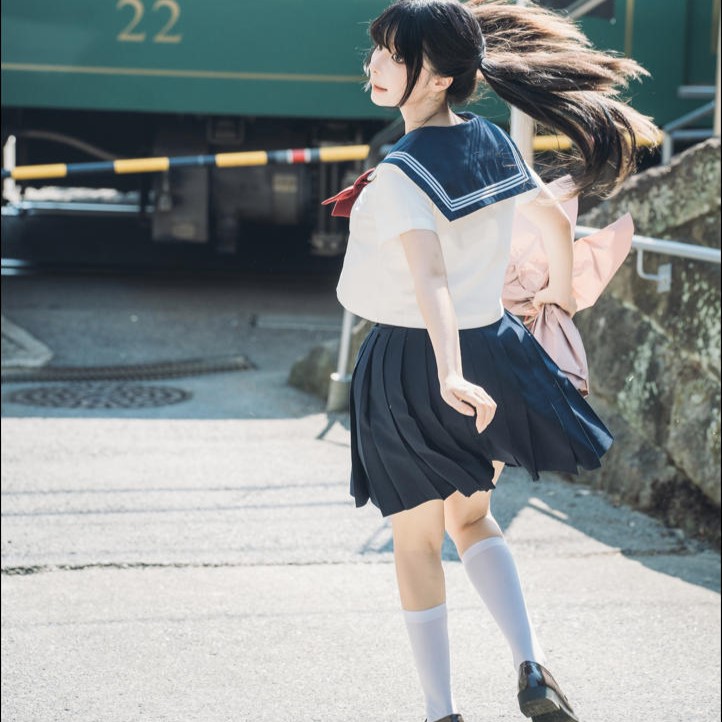}
        \caption{original image}
        \label{fig:app-visuala}
    \end{subfigure}\hfill
    \begin{subfigure}[t]{0.325\textwidth}
        \centering
        \includegraphics[width=\linewidth]{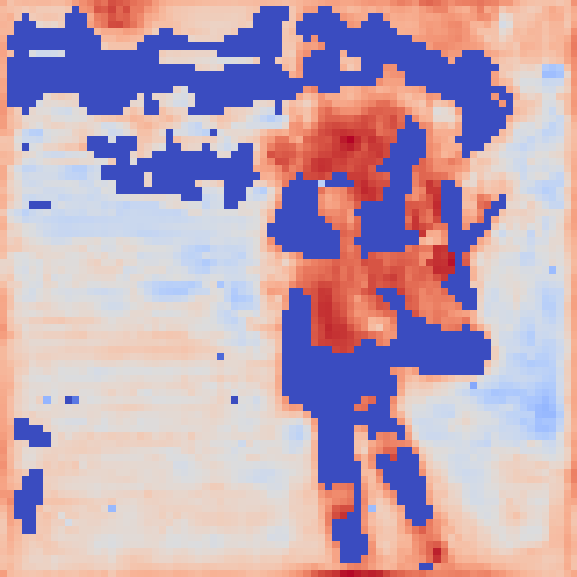}
        \caption{positive prompt: sailor uniform}
        \label{fig:visualposi}
    \end{subfigure}\hfill
    \begin{subfigure}[t]{0.325\textwidth}
        \centering
        \includegraphics[width=\linewidth]{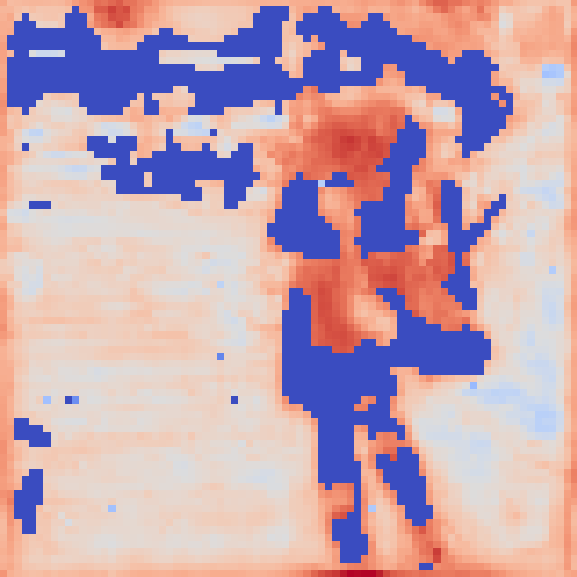}
        \caption{negative prompt: black sailor uniform}
        \label{fig:visualnega}
    \end{subfigure}
    \caption{EAM heatmaps for positive prompt \textit{sailor uniform} and negative prompt \textit{black sailor uniform}. Background-dominant activations indicate effective suppression of unseen negative prompts.}
    \label{fig:eam-neg}
\end{figure*}

\section{Impact of Global Negative Vocabulary}

During training, we observe that the composition and quality of the global negative vocabulary have a noticeable impact on performance, particularly for rare categories (AP$_r$).
Varying the negative prompt set—through different sampling strategies, vocabulary sizes, or semantic distributions—results in fluctuations of approximately $\pm 0.8\%$ in AP$_r$ on the LVIS minival split.
In contrast, the effect on overall AP as well as AP$_c$ and AP$_f$ is relatively limited, with variations within $\pm 0.2\%$.

This behavior suggests that rare-category representations in the CLIP embedding space are inherently more fragile and sensitive to interference from negative prompts.
When negative samples are semantically close to rare positives or occupy nearby regions in the embedding space, they can induce stronger gradient conflicts during contrastive alignment, disproportionately impairing the model’s ability to discriminate long-tail classes.
Frequent and common categories, which are more densely covered during vision--language pre-training, exhibit greater robustness to such perturbations.

We hypothesize that an effective negative vocabulary should occupy a ``sweet spot'' in the CLIP feature space: sufficiently separated from the positive (LVIS) distribution to suppress false activations, yet not so distant that the negatives become uninformative and yield weak or noisy gradients.
Negative sets that are overly similar to positives may lead to excessive suppression and hinder rare-category learning, while overly distant negatives may fail to provide meaningful discriminative supervision.
Identifying such a balanced negative distribution could further improve performance on LVIS, particularly for long-tail categories.

At present, however, there is no standardized metric or principled methodology to quantify the ``quality'' or ``difficulty'' of a global negative vocabulary in open-vocabulary detection.
Developing reliable criteria or adaptive strategies for negative vocabulary construction—such as embedding-aware sampling, online hard-negative mining, or dynamic vocabulary curation—remains an open challenge and a promising direction for future work.

\section{More Visualization Examples}
\label{app:vis_qualitative}

Figure~\ref{fig:emorevisual} and ~\ref{fig:morevisual} shows additional zero-shot detection and segmentation results of ExpAlign on diverse scenes with multi-object and detailed text prompts. The model demonstrates strong open-vocabulary grounding and precise instance masks across novel categories and complex compositions.

\begin{figure*}[!htbp]
    \centering
    \begin{subfigure}[t]{0.49\textwidth}
        \centering
        \includegraphics[width=\linewidth]{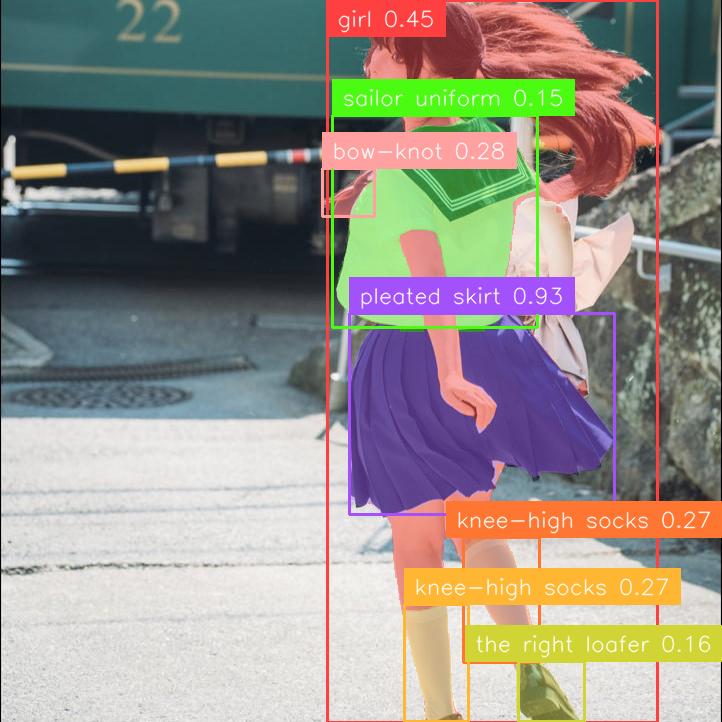}
        \caption{}
        \label{fig:emorevisuala}
    \end{subfigure}\hfill
    \begin{subfigure}[t]{0.49\textwidth}
        \centering
        \includegraphics[width=\linewidth]{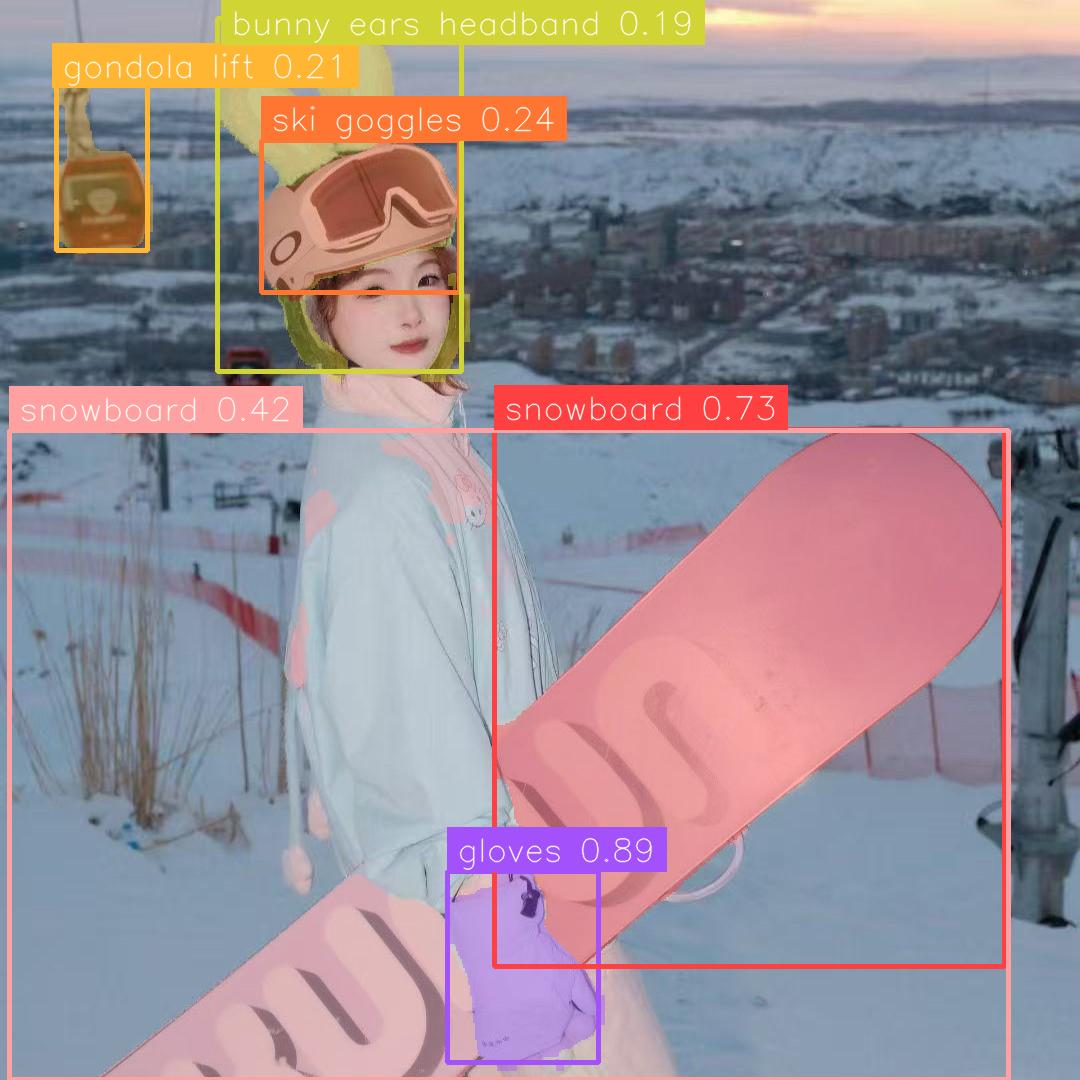}
        \caption{}
        \label{fig:emorevisualb}
    \end{subfigure}\\
    \begin{subfigure}[t]{0.49\textwidth}
        \centering
        \includegraphics[width=\linewidth]{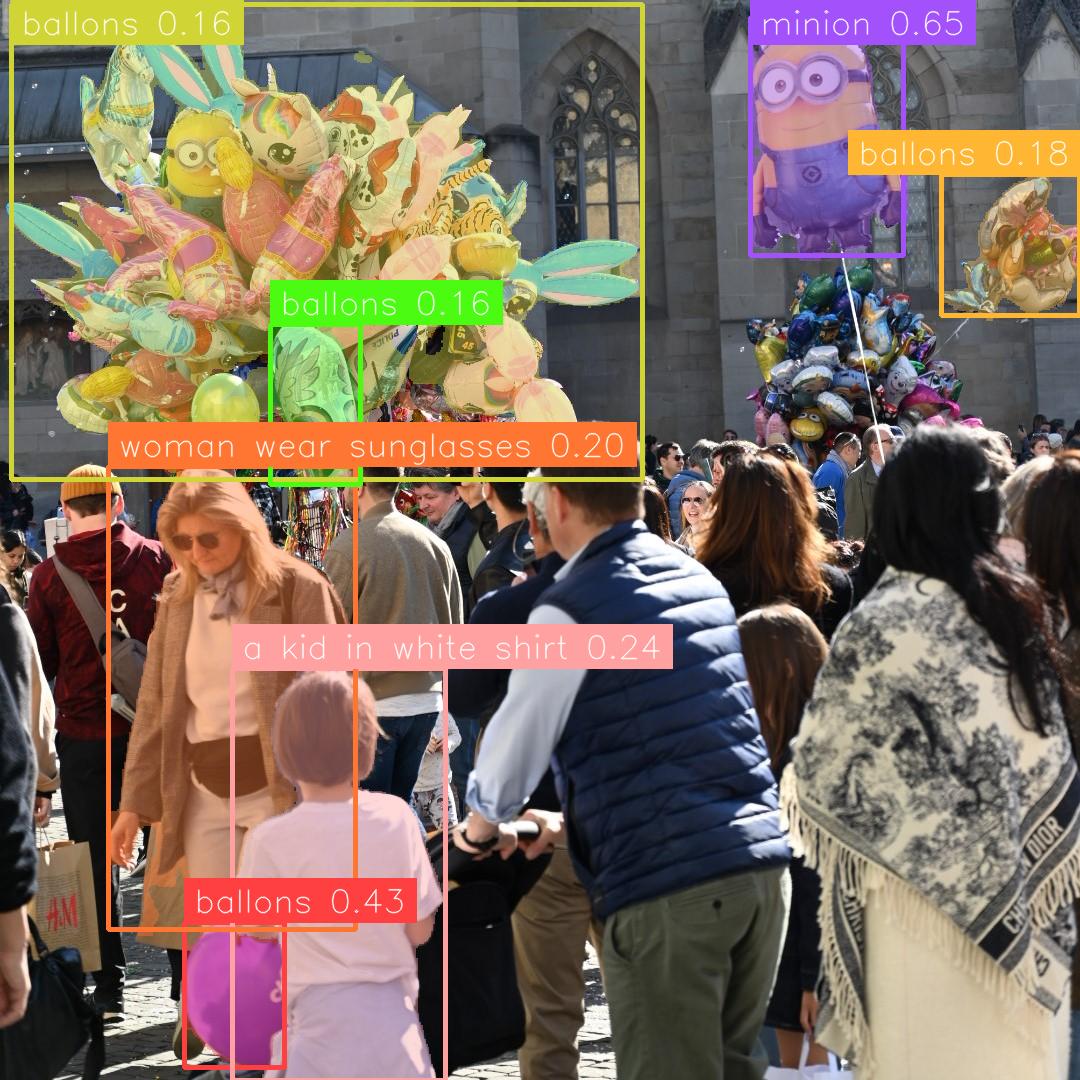}
        \caption{}
        \label{fig:morevisualc}
    \end{subfigure}\hfill
    \begin{subfigure}[t]{0.49\textwidth}
        \centering
        \includegraphics[width=\linewidth]{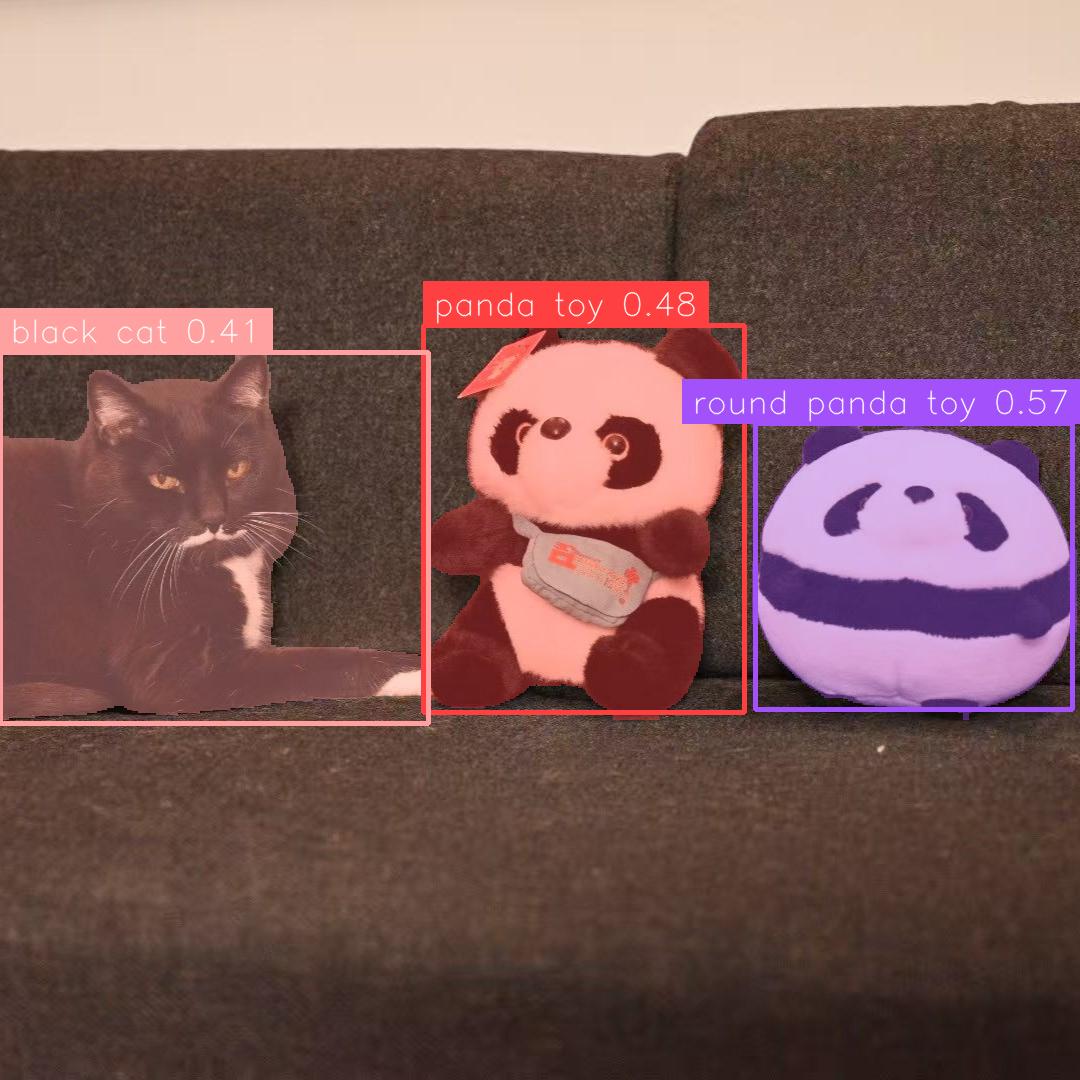}
        \caption{}
        \label{fig:morevisuald}
    \end{subfigure}
    \caption{(a) prompts: girl, sailor uniform, the right loafer, bow-knot, knee-high socks, pleated skirt. (b) prompts: snowboard, ski goggles, gondola lift, gloves, bunny ears headband. (c) prompts: minion, ballons, a kid in white shirt, woman wear sunglasses. (d) prompts: black cat, panda toy, round panda toy. Zoom in for better visual effect.}
    \label{fig:emorevisual}
\end{figure*}

\begin{figure*}[t]
    \centering
    \begin{subfigure}[t]{0.49\textwidth}
        \centering
        \includegraphics[width=\linewidth]{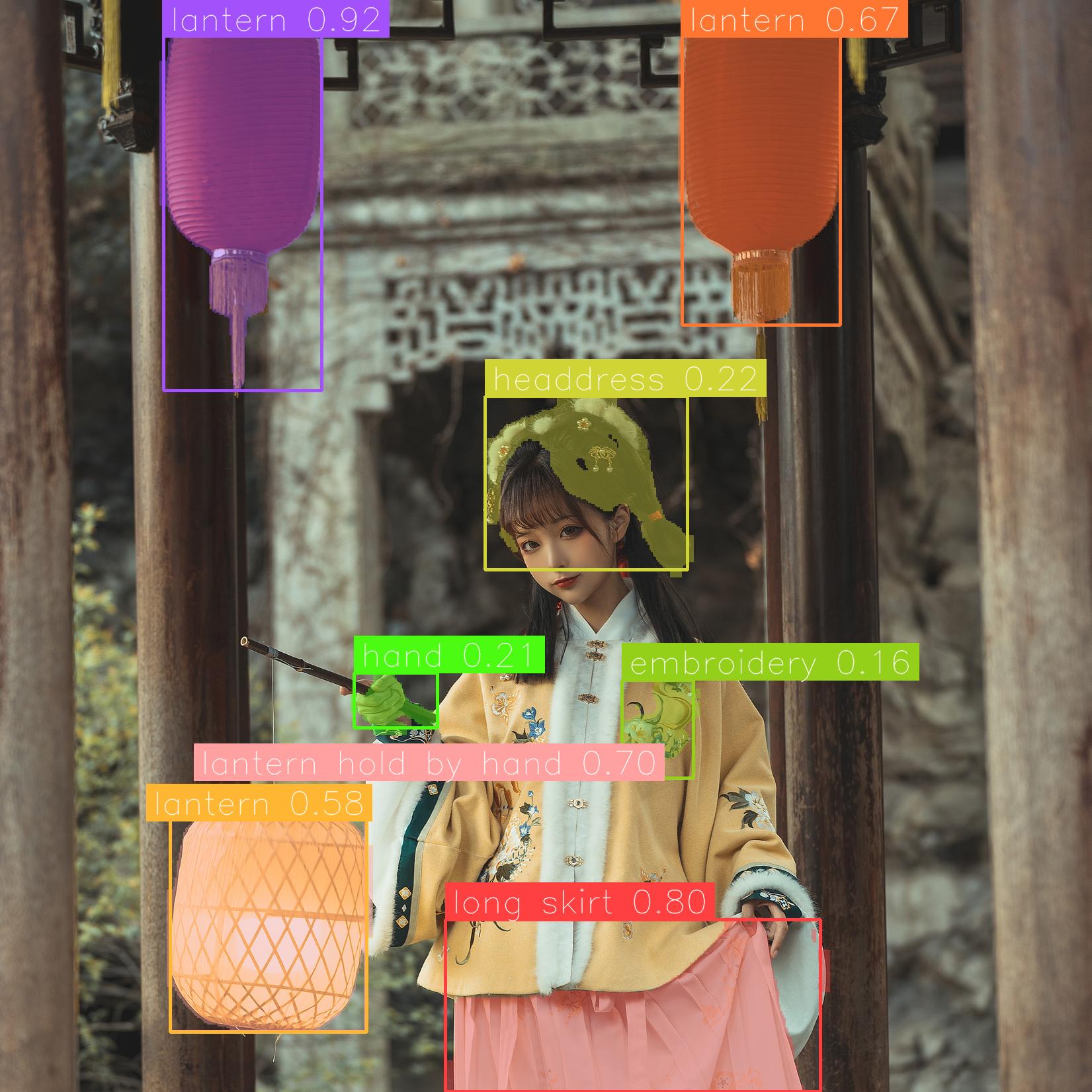}
        \caption{}
        \label{fig:morevisuala}
    \end{subfigure}\hfill
    \begin{subfigure}[t]{0.49\textwidth}
        \centering
        \includegraphics[width=\linewidth]{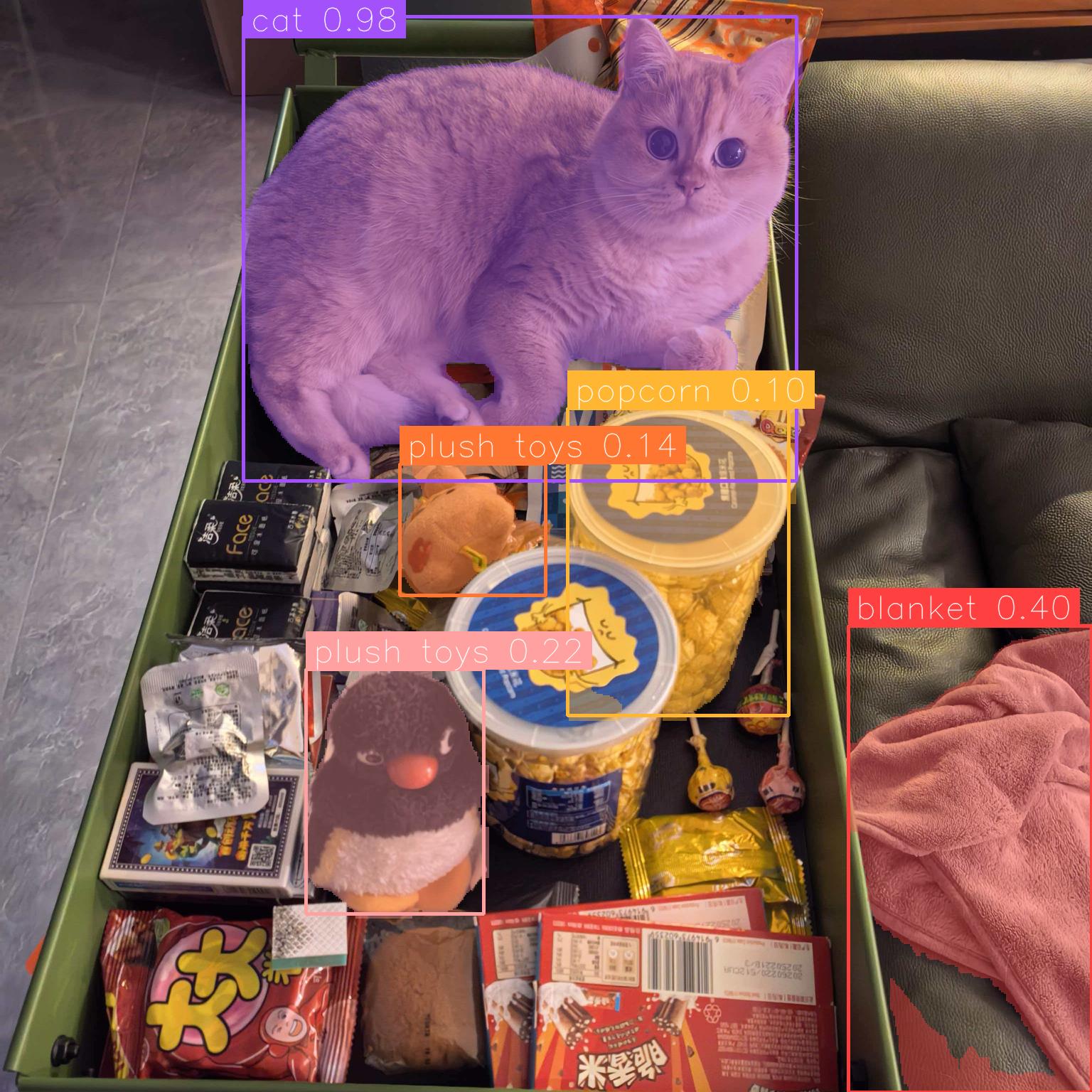}
        \caption{}
        \label{fig:morevisualb}
    \end{subfigure}\\
    \begin{subfigure}[t]{0.49\textwidth}
        \centering
        \includegraphics[width=\linewidth]{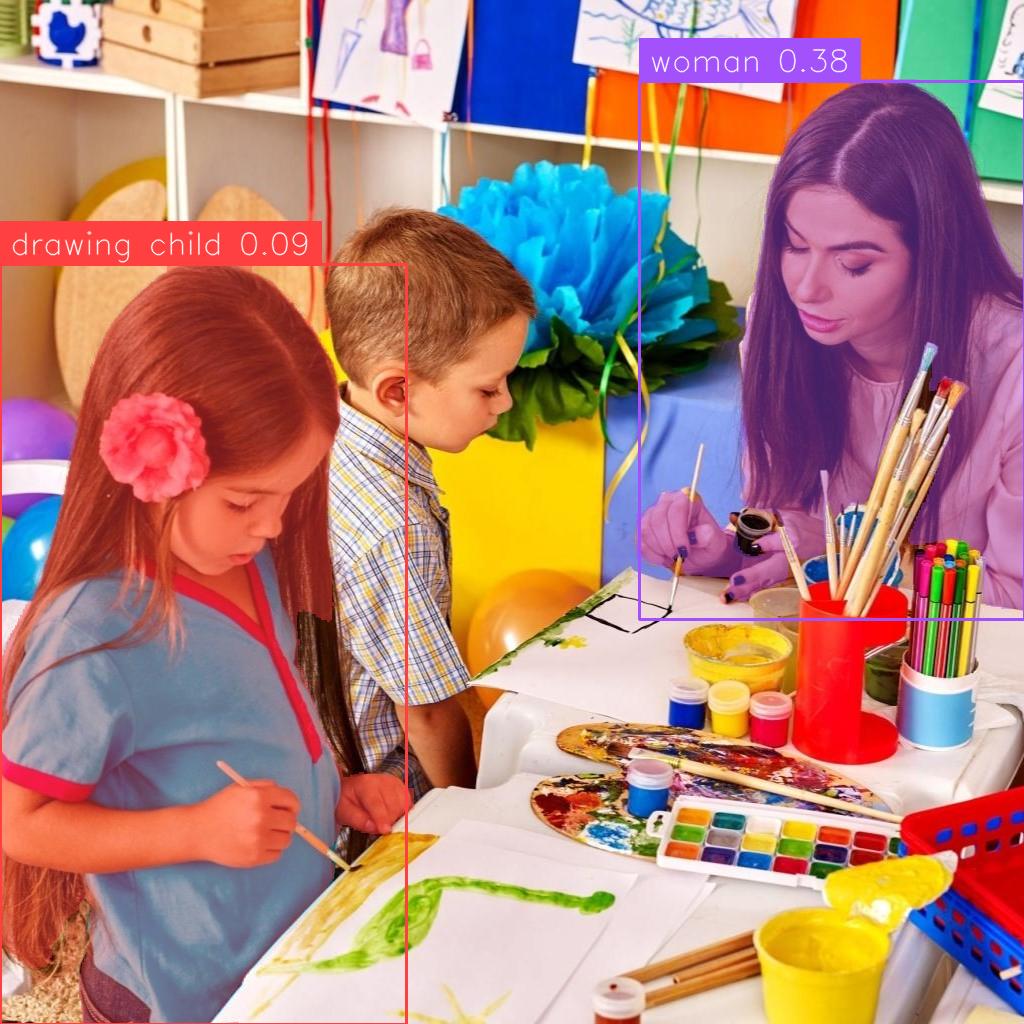}
        \caption{}
        \label{fig:morevisualc}
    \end{subfigure}\hfill
    \begin{subfigure}[t]{0.49\textwidth}
        \centering
        \includegraphics[width=\linewidth]{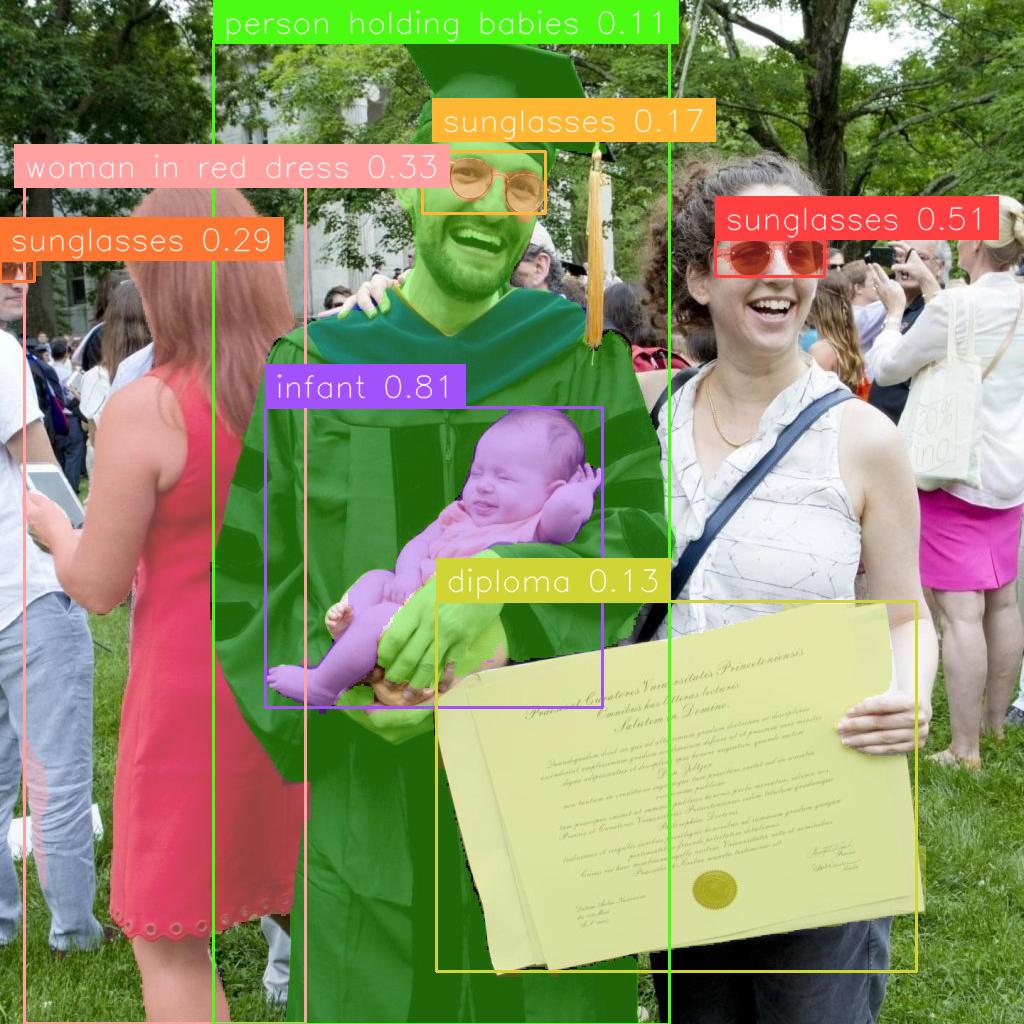}
        \caption{}
        \label{fig:morevisuald}
    \end{subfigure}
    \caption{(a) prompts: lantern, lantern hold by hand, hand, headdress, embroidery, long skirt. (b) prompts: cat, popcorn, blacket, plush toys. (c) prompts: drawing child, woman. (d) prompts: diploma, person holding babies, woman in red dress, infant, sunglasses. Zoom in for better visual effect.}
    \label{fig:morevisual}
\end{figure*}

\end{document}